%% file: nmr14.tex
\documentclass[a4,a4wide]{article}
\usepackage{aaai}
\usepackage{times}
\usepackage{balance}

\usepackage[mathscr]{eucal}

\usepackage{amsmath,amssymb,latexsym}

\usepackage{braket}

\usepackage{booktabs}

\usepackage[neverdecrease]{paralist}

\usepackage[disable]{todonotes}

\usepackage{stmaryrd}

\usepackage{comment}

\usepackage{environ}

\newcommand{\br}[1]{(#1)}

\newcommand{\tpl}[1]{\br{#1}}

\newcommand{\seq}[1]{\langle #1 \rangle}
\newcommand{\Seq}[1]{\left\langle #1 \right\rangle}

\newcommand{\lng}{n}
\newcommand{\lia}{i}
\newcommand{\lib}{j}
\newcommand{\lic}{k}

\newcommand{\sqbbr}[1]{\llbracket\hspace{.2ex}#1\hspace{.2ex}\rrbracket}
\newcommand{\Sqbbr}[1]{\left\llbracket\hspace{.2ex}#1\hspace{.2ex}\right\rrbracket}

\newcommand{\atm}{p}
\newcommand{\atma}{p}
\newcommand{\atmb}{q}
\newcommand{\atmc}{r}
\newcommand{\atmd}{s}
\newcommand{\atms}{\mathscr{A}}
\newcommand{\atmst}[1]{\ensuremath{\mathsf{#1}}}

\newcommand{\olit}{l}
\newcommand{\olits}{\mathscr{L}}
\newcommand{\lopp}[1]{\lnot #1}
\newcommand{\lcmp}[1]{\lpnot #1}

\newcommand{\lit}{L}
\newcommand{\slit}{S}
\newcommand{\slitp}[1][\slit]{#1^+}
\newcommand{\slitn}[1][\slit]{#1^-}

\newcommand{\lpnot}[1][\!]{\mathop{\sim#1}}
\newcommand{\lpif}{\leftarrow}

\newcommand{\rl}{\pi}
\newcommand{\rla}{\pi}
\newcommand{\rlb}{\sigma}
\newcommand{\rlc}{\rho}

\newcommand{\hrl}[1][\rl]{\textrm{\sf \small H}_{#1}}
\newcommand{\hrla}{\hrl[\rla]}
\newcommand{\hrlb}{\hrl[\rlb]}
\newcommand{\hrlc}{\hrl[\rlc]}

\newcommand{\brl}[1][\rl]{\textrm{\sf \small B}_{#1}}
\newcommand{\brla}{\brl[\rla]}
\newcommand{\brlb}{\brl[\rlb]}

\newcommand{\prg}{P}
\newcommand{\prga}{P}
\newcommand{\prgb}{Q}

\newcommand{\prgu}{U}
\newcommand{\prgv}{V}

\newcommand{\dprg}{\text{\bf \em P}}

\newcommand{\all}[1]{\mathsf{all}(#1)}

\newcommand{\expone}[1]{#1^{\dagger}}
\newcommand{\exptwo}[1]{#1^{\ddagger}}

\newcommand{\defl}[1]{\mathsf{def}(#1)}

\newcommand{\rej}[3]{\mathsf{rej}^{#1}_{#2}(#3)}
\newcommand{\rejrd}[1]{\rej{}{\geq}{#1}}
\newcommand{\rejrds}[1]{\rej{\lnot}{>}{#1}}
\newcommand{\rejws}[2][\lm]{\rej{}{#1}{#2}}
\newcommand{\rejwss}[2][\lm]{\rej{\lnot}{#1}{#2}}

\newcommand{\rem}[2][]{\mathsf{rem}_{#1}(#2)}

\newcommand{\defrddef}{
	\defl{\dprg, \twi} &= \{
		(\lpnot \olit.)
		|
		\olit \in \olits \\
		& \hspace{1.25cm}\land
		\lnot (
			\exists \rl \in \all{\dprg}
			:
			\hrl = \olit
			\land
			\twi \ent \brl
		)
	\}
}

\newcommand{\rejrddef}{
	\rejrd{\dprg, \twi} &= \{
		\rla \in \prg_\lia
		|
		\lia < \lng
		\land
		\exists \lib \geq \lia \; \exists \rlb \in \prg_\lib
		:
		\hrla = \lpnot \hrlb\\
		& \hspace{5.25cm}\land
		\twi \ent \brlb
	\}
}

\newcommand{\rejrdsdef}{
	\rejrds{\dprg, \slit}
	&=
	\{
		\rla \in \prg_\lia
		|
		\lia < \lng
		\land
		\exists \lib > \lia
		\;
		\exists \rlb \in \prg_\lib
		:
		\hrlb \in \con{\hrla}
		\\
		& \hspace{5.0cm} \land
		\brlb \subseteq \slit
	\}
}

\newcommand{\rejwsdef}{
	\rejws{\dprg, \twi} &= \{
		\rla \in \prg_\lia
		|
		\lia < \lng
		\land
		\exists \lib > \lia \; \exists \rlb \in \prg_\lib
		:
		\hrla = \lpnot \hrlb\\
		& \hspace{2.6cm}\land
		\twi \ent \brlb
		\land
		\lm(\hrlb) > \lmmax(\brlb)
	\}
}

\newcommand{\con}[1]{\overline{#1}}

\newcommand{\rejwssdef}{
	\rejwss{\dprg, \twi} = \{
		\rla \in \prg_\lia
		|
		\lia < \lng
		&\land
		\exists \lib > \lia \; \exists \rlb \in \prg_\lib
		:
		\hrlb \in \con{\hrla}
		\\
		&\land
		\twi \ent \brlb
		\land
		\lmmin\!\left(\con{\hrla}\right) > \lmmax(\brlb)
	\}
	}

\newcommand{\twi}{J}

\newcommand{\twib}{I}

\newcommand{\twiall}[1][\twi]{#1^*}

\newcommand{\least}[1]{\mathsf{least}(#1)}
\newcommand{\Least}[1]{\mathsf{least}\left(#1\right)}

\newcommand{\prefix}[1]{\textsf{\small #1}}
\newcommand{\prefm}[1]{\ensuremath{\scriptscriptstyle \mathsf{#1}}}
\newcommand{\prefixm}[1]{\raisebox{-1pt}{\prefm{#1}}}

\newcommand{\mSM}{\prefixm{SM}}
\newcommand{\modsm}[1]{\sqbbr{#1}_{\mSM}}

\renewcommand{\S}[2][S]{\prefix{#1}\protect\nobreakdash#2\hspace{0pt}}
\newcommand{\mS}{\prefixm{S}}
\newcommand{\mods}[1]{\sqbbr{#1}_{\mS}}
\newcommand{\Mods}[1]{\Sqbbr{#1}_{\mS}}

\newcommand{\RD}[1]{\prefix{RD}\protect\nobreakdash#1\hspace{0pt}}

\newcommand{\mRD}{\prefixm{RD}}
\newcommand{\modrd}[1]{\sqbbr{#1}_{\mRD}}
\newcommand{\modrds}[1]{\sqbbr{#1}_{\mRD}^{\lnot}}

\newcommand{\WS}[1]{\prefix{WS}\protect\nobreakdash#1\hspace{0pt}}

\newcommand{\mWS}{\prefixm{WS}}
\newcommand{\modws}[1]{\sqbbr{#1}_{\mWS}}
\newcommand{\modwss}[1]{\sqbbr{#1}_{\mWS}^{\lnot}}

\newcommand{\mymodels}{\mathrel\mid\joinrel=}
\newcommand{\ent}{\mymodels}
\newcommand{\nent}{\not\ent}

\newcommand{\imcon}[1][\prg]{T_{#1}}
\newcommand{\imconrds}[1][\dprg, \twi]{T_{#1}}

\newcommand{\lm}{\ell}
\newcommand{\lmmax}{\ell^{\uparrow}}
\newcommand{\lmmin}{\ell^{\downarrow}}

\newcommand{\NP}{\textsf{NP}}
\newcommand{\coNP}{\textsf{coNP}}

 \usepackage{amsthm}
 \newtheorem{theorem}{Theorem}[section]
 \newtheorem{lemma}[theorem]{Lemma}
 \newtheorem{proposition}[theorem]{Proposition}

 \newtheorem{definition}[theorem]{Definition}
 \newtheorem{example}[theorem]{Example}

\NewEnviron{proofmain}{%
	\begin{proof}
		\BODY
	\end{proof}
}
\excludecomment{proofmain}

\usepackage{ifthen}
\newcommand{\brifnotempty}[1]{\ifthenelse{\equal{#1}{}}{}{ \br{#1}}}
\newenvironment{lemma*}[2][]
	{\pagebreak[2] \par \noindent \textbf{Lemma~\ref{#2}}\brifnotempty{#1}.\it}{\par}
\newenvironment{theorem*}[2][]
	{\pagebreak[2] \par \noindent \textbf{Theorem~\ref{#2}}\brifnotempty{#1}.\it}{\par}
\newenvironment{proposition*}[2][]
	{\pagebreak[2] \par \noindent \textbf{Proposition~\ref{#2}}\brifnotempty{#1}.\it}{\par}
\newenvironment{corollary*}[2][]
	{\pagebreak[2] \par \noindent \textbf{Corollary~\ref{#2}}\brifnotempty{#1}.\it}{\par}

\newenvironment{textitem}
	{\setlength{\pltopsep}{.7ex}\setlength{\plitemsep}{.7ex}\begin{compactitem}}
	{\end{compactitem}}

\begin{document}
%

\nocopyright

\title{On Strong and Default Negation in Logic Program Updates (Extended Version)}
\author{Martin Slota\\CENTRIA\\New University of Lisbon
\And Martin Bal\'{a}\v{z}\\Faculty of Mathematics, Physics and Informatics\\Comenius University
\And Jo\~{a}o Leite\\CENTRIA\\New University of Lisbon}

\maketitle

\begin{abstract}
	\input{abstract}
\end{abstract}

\section{Introduction}

\label{sec:introduction}

\input{introduction}

\section{Background}

\label{sec:background}

In this section we introduce the necessary technical background and generalise
the well-supported semantics \cite{Fages1991} to the class of extended
programs.

\subsection{Logic Programs}

\label{sec:logic programs}

\input{logic-programs}

\subsection{Well-supported Models for Extended Programs}

\label{sec:semantics static}

\input{semantics-static}

\subsection{Rule Updates}

\label{sec:rule updates}

\input{rule-updates}

\section{Direct Support for Strong Negation in Rule Updates}

\label{sec:semantics dynamic}

\input{semantics-dynamic}

\section{Properties}

\label{sec:properties}

\input{properties}

\section{Concluding Remarks}

\label{sec:discussion}

\input{discussion}

\section*{Acknowledgments}
Jo{\~a}o Leite was partially supported by Funda\c{c}\~{a}o para a Ci\^{e}ncia e a
Tecnologia under project ``ERRO -- Efficient Reasoning with Rules and
Ontologies'' ({PTDC}/{EIA}-{CCO}/{121823}/{2010}). Martin Slota was  partially supported by Funda\c{c}\~{a}o para a Ci\^{e}ncia e a Tecnologia under project ``ASPEN -- Answer Set Programming with BoolEaN Satisfiability'' (PTDC/EIA-CCO/110921/2009).
The collaboration between the co-authors resulted from the Slovak--Portuguese bilateral
project ``ReDIK -- Reasoning with Dynamic Inconsistent Knowledge'',
supported by APVV agency under SK-PT-0028-10 and by Funda\c{c}\~{a}o para a Ci\^{e}ncia e a Tecnologia (FCT/2487/3/6/2011/S).


\bibliographystyle{aaai}
\bibliography{nmr14}
\balance
\appendix


\section{Proofs}


\label{app:proofs}

\input{proofs}

\balance
\label{lastpage}
\end{document}

%% file: abstract.tex
Existing semantics for answer-set program updates fall into two categories:
either they consider only \emph{strong negation} in heads of rules, or they
primarily rely on \emph{default negation} in heads of rules and optionally
provide support for strong negation by means of a syntactic transformation.

In this paper we pinpoint the limitations of both these approaches and argue
that both types of negation should be first-class citizens in the context of
updates. We identify principles that plausibly constrain their interaction but
are not simultaneously satisfied by any existing rule update semantics. Then
we extend one of the most advanced semantics with direct support for strong
negation and show that it satisfies the outlined principles as well as a
variety of other desirable properties.


%% file: introduction.tex
The increasingly common use of rule-based knowledge representation languages
in highly dynamic and information-rich contexts, such as the Semantic Web
\cite{Berners-Lee2001}, requires standardised support for updates of knowledge
represented by rules. Answer-set programming \cite{Gelfond1988,Gelfond1991}
forms the natural basis for investigation of rule updates, and various
approaches to answer-set program updates have been explored throughout the
last 15 years
\cite{Leite1997,Alferes98,Alferes2000,Eiter2002,Leite2003,Sakama2003,Alferes2005,Banti2005,Zhang2006,Sefranek2006,Delgrande2007,Osorio2007,Sefranek2011,Krumpelmann2012}.


The most straightforward kind of conflict arising between an original rule and
its update occurs when the original conclusion logically contradicts the newer
one. Though the technical realisation and final result may differ
significantly, depending on the particular rule update semantics, this kind of
conflict is resolved by letting the newer rule prevail over the older one.
Actually, under most semantics, this is also the \emph{only} type of conflict
that is subject to automatic resolution
\cite{Leite1997,Alferes2000,Eiter2002,Alferes2005,Banti2005,Delgrande2007,Osorio2007}.


From this perspective, allowing for both \emph{strong} and \emph{default
negation} to appear in heads of rules is essential for an expressive and
universal rule update framework \cite{Leite2003}. While strong negation is the
natural candidate here, used to express that an atom \emph{becomes explicitly
false}, default negation allows for more fine-grained control: the atom only
\emph{ceases to be true}, but its truth value may not be known after the
update. The latter also makes it possible to move between any pair of
epistemic states by means of updates, as illustrated in the following example:

\begin{example}
	[Railway crossing \cite{Leite2003}]
	\label{ex:railway crossing}
	Suppose that we use the following logic program to choose an action at a
	railway crossing:
	\begin{align*}
		\atmst{cross} &\lpif \lnot \atmst{train}.
		&
		\atmst{wait} &\lpif \atmst{train}.
		&
		\atmst{listen} &\lpif \lpnot \atmst{train}, \lpnot \lnot \atmst{train}.
	\end{align*}
	The intuitive meaning of these rules is as follows: one should \atmst{cross}
	if there is evidence that no train is approaching; \atmst{wait} if there is
	evidence that a train is approaching; \atmst{listen} if there is no such
	evidence.

	Consider a situation where a train is approaching, represented by the fact
	$(\atmst{train}.)$. After this train has passed by, we want to update our
	knowledge to an epistemic state where we lack evidence with regard to the
	approach of a train. If this was accomplished by updating with the fact
	$(\lnot \atmst{train}.)$, we would cross the tracks at the subsequent state,
	risking being killed by another train that was approaching. Therefore, we
	need to express an update stating that all past evidence for an atom is to
	be removed, which can be accomplished by allowing default negation in heads
	of rules. In this scenario, the intended update can be expressed by the fact
	$(\lpnot \atmst{train}.)$.
\end{example}

With regard to the support of negation in rule heads, existing rule update
semantics fall into two categories: those that only allow for strong negation,
and those that primarily consider default negation. As illustrated above, the
former are unsatisfactory as they render many belief states unreachable by
updates. As for the latter, they optionally provide support for strong
negation by means of a syntactic transformation.

Two such transformations are known from the literature, both of them based on
the principle of coherence: if an atom $\atm$ is true, its strong negation
$\lnot \atm$ cannot be true simultaneously, so $\lpnot \lnot \atm$ must be
true, and also vice versa, if $\lnot \atm$ is true, then so is $\lpnot \atm$.
The first transformation, introduced in \cite{Alferes1996}, encodes this
principle directly by adding, to both the original program and its update, the
following two rules for every atom $\atm$:
\begin{align*}
	\lpnot \lnot \atm &\lpif \atm.
	&
	\lpnot \atm &\lpif \lnot \atm.
\end{align*}
This way, every conflict between an atom $\atm$ and its strong negation $\lnot
\atm$ directly translates into two conflicts between the objective literals
$\atm$, $\lnot \atm$ and their default negations. However, the added rules
lead to undesired side effects that stand in direct opposition with basic
principles underlying updates. Specifically, despite the fact that the empty
program does not encode any change in the modelled world, the stable models
assigned to a program may change after an update by the empty program.

This undesired behaviour is addressed in an alternative transformation from
\cite{Leite2003} that encodes the coherence principle more carefully.
%
%
Nevertheless, this transformation also leads to undesired consequences, as
demonstrated in the following example:

\begin{example}
	[Faulty sensor]
	\label{ex:faulty sensor}
	Suppose that we collect data from sensors and, for security reasons,
	multiple sensors are used to supply information about the critical fluent
	$\atm$. In case of a malfunction of one of the sensors, we may end up with
	an inconsistent logic program consisting of the following two facts:
	\begin{align*}
		\atm & .
		&
		\lnot \atm & .
	\end{align*}
	At this point, no stable model of the program exists and action needs to be
	taken to find out what is wrong. If a problem is found in the sensor that
	supplied the first fact $(\atm.)$, after the sensor is repaired, this
	information needs to be reset by updating the program with the fact $(\lpnot
	\atm.)$. Following the universal pattern in rule updates, where recovery
	from conflicting states is always possible, we expect that this update
	is sufficient to assign a stable model to the updated program.  However, the
	transformational semantics for strong negation defined in \cite{Leite2003}
	still does not provide any stable model -- we remain without a valid
	epistemic state when one should in fact exist.
\end{example}

In this paper we address the issues with combining strong and default negation
in the context of rule updates. Based on the above considerations, we
formulate a generic desirable principle that is violated by the existing
approaches. Then we show how two distinct definitions of one of the most
well-behaved rule update semantics \cite{Alferes2005,Banti2005} can be
equivalently extended with support for strong negation. The resulting
semantics not only satisfies the formulated principle, but also retains the
formal and computational properties of the original semantics. More
specifically, our main contributions are as follows:
\begin{itemize}

	\item based on Example~\ref{ex:faulty sensor}, we introduce the \emph{early
		recovery principle} that captures circumstances under which a stable model
		after a rule update should exist;

	\item we extend the \emph{well-supported semantics for rule updates}
		\cite{Banti2005} with direct support for strong negation;

	\item we define a fixpoint characterisation of the new semantics, based on
		the \emph{refined dynamic stable model} semantics for rule updates
		\cite{Alferes2005};
		

	\item we show that the defined semantics enjoy the early recovery principle
		as well as a range of desirable properties for rule updates known from the
		literature.
\end{itemize}

This paper is organised as follows: In Sect.~\ref{sec:background} we
present the syntax and semantics of logic programs, generalise the
well-supported semantics from the class of normal programs to extended ones
and define the rule update semantics from \cite{Alferes2005,Banti2005}. Then,
in Sect.~\ref{sec:semantics dynamic}, we formally establish the early recovery
principle, define the new rule update semantics for strong negation and show
that it satisfies the principle. In Sect.~\ref{sec:properties} we introduce
other established rule update principles and show that the proposed semantics
satisfies them. We discuss our findings and conclude in
Sect.~\ref{sec:discussion}.\footnote{%
	The proofs of all propositions and theorems can be found in Appendix
	\ref{app:proofs}.
}


%% file: logic-programs.tex
In the following we present the syntax of non-disjunctive logic programs with
both strong and default negation in heads and bodies of rules, along with the
definition of stable models of such programs from \cite{Leite2003} that is
equivalent to the original definitions based on reducts
\cite{Gelfond1988,Gelfond1991,Inoue1998}. Furthermore, we define an
alternative characterisation of the stable model semantics: the well-supported
models of normal logic programs \cite{Fages1991}.

We assume that a countable set of propositional atoms $\atms$ is given and
fixed. An \emph{objective literal} is an atom $\atm \in
\atms$ or its strong negation $\lnot \atm$. We denote the set of all objective
literals by $\olits$. A \emph{default literal} is an objective literal
preceded by $\lpnot[]$ denoting default negation. A \emph{literal} is either
an objective or a default literal. We denote the set of all literals by
$\olits^*$. As a convention, double negation is absorbed, so that $\lnot \lnot
\atm$ denotes the atom $\atm$ and $\lpnot \lpnot \olit$ denotes the objective
literal $\olit$. Given a set of literals $\slit$, we introduce the following
notation: $\slitp = \Set{\olit \in \olits | \olit \in \slit}$, $\slitn
= \Set{\olit \in \olits | \lpnot \olit \in \slit}$, $\lpnot \slit
= \Set{\lpnot \lit | \lit \in \slit}$.

An \emph{extended rule} is a pair $\rl = \tpl{\hrl, \brl}$ where $\hrl$ is a
literal, referred to as the \emph{head of $\rl$}, and $\brl$ is a finite set
of literals, referred to as the \emph{body of $\rl$}.
Usually we write $\rl$ as
$
	(\hrl \lpif \brl^+, \lpnot \brl^-.)
$.
A \emph{generalised rule} is an extended rule that contains no occurrence of
$\lnot$, i.e., its head and body consist only of atoms and their default
negations. A \emph{normal rule} is a generalised rule that has an atom in the
head. A \emph{fact} is an extended rule whose body is empty and a
\emph{tautology} is any extended rule $\rl$ such that $\hrl \in \brl$. An
\emph{extended (generalised, normal) program} is a set of extended
(generalised, normal) rules.

An \emph{interpretation} is a consistent subset of the set of objective
literals, i.e., a subset of $\olits$ does not contain both $\atm$ an $\lnot
\atm$ for any atom $\atm$. The satisfaction of an objective literal $\olit$,
default literal $\lpnot \olit$, set of literals $\slit$, extended rule $\rl$
and extended program $\prg$ in an interpretation $\twi$ is defined in the
usual way: $\twi \ent \olit$ iff $\olit \in \twi$; $\twi \ent \lpnot \olit$
iff $\olit \notin \twi$; $\twi \ent \slit$ iff $\twi \ent \lit$ for all $\lit
\in \slit$; $\twi \ent \rl$ iff $\twi \ent \brl$ implies $\twi \ent \hrl$;
$\twi \ent \prg$ iff $\twi \ent \rl$ for all $\rl \in \prg$.
Also, $\twi$ is a \emph{model of $\prg$} if $\twi \ent \prg$, and $\prg$ is
\emph{consistent} if it has a model.

\begin{definition}
	[Stable model]
	Let $\prg$ be an extended program. The set $\modsm{\prg}$ of \emph{stable
	models of $\prg$} consists of all interpretations $\twi$ such that
	\[
		\twiall = \least{
			\prg
			\cup
			\defl{\twi}
		}
	\]
	where $\defl{\twi} = \Set{\lpnot \olit. | \olit \in \olits \setminus
	\twi}$, $\twiall = \twi \cup \lpnot (\olits \setminus \twi)$ and
	$\least{\cdot}$ denotes the least model of the argument program in which all
	literals are treated as propositional atoms.
\end{definition}

A \emph{level mapping} is a function that maps every atom to a natural number.
Also, for any default literal $\lpnot \atm$, where $\atm \in \atms$, and
finite set of atoms and their default negations $\slit$, $\lm(\lpnot \atm) =
\lm(\atm)$, $\lmmin(\slit) = \min \set{\lm(\lit) | \lit \in \slit}$ and
$\lmmax(\slit) = \max \set{\lm(\lit) | \lit \in \slit}$.

\begin{definition}
	[Well-supported model of a normal program]
	Let $\prg$ be a normal program and $\lm$ a level mapping. An interpretation
	$\twi \subseteq \atms$ is a \emph{well-supported model of $\prg$ w.r.t.\
	$\lm$} if the following conditions are satisfied:
	\begin{enumerate}[1.]
		\item $\twi$ is a model of $\prg$;

		\item For every atom $\atm \in \twi$ there exists a rule $\rl \in \prg$
			such that
			\[
				\hrl = \atm \land \twi \ent \brl \land \lm(\hrl) > \lmmax(\brl)
				\enspace.
			\]
	\end{enumerate}
	The set $\modws{\prg}$ of \emph{well-supported models of $\prg$} consists of
	all interpretations $\twi \subseteq \atms$ such that $\twi$ is a
	well-supported model of $\prg$ w.r.t.\ some level mapping.
\end{definition}

As shown in \cite{Fages1991}, well-supported models coincide with stable
models:

\begin{proposition}[\hspace{-0.07ex}\cite{Fages1991}]
	\label{prop:ws coincide with sm}
	Let $\prg$ be a normal program. Then, $\modws{\prg} = \modsm{\prg}$.\end{proposition}


%% file: semantics-static.tex
The well-supported models defined in the previous section for normal logic
programs can be generalised in a straightforward manner to deal with strong
negation while maintaining their tight relationship with stable models (c.f.\
Proposition~\ref{prop:ws coincide with sm}). This will come useful in
Subsect.~\ref{sec:rule updates} and Sect.~\ref{sec:semantics dynamic} when we
discuss adding support for strong negation to semantics for rule updates.

We extend level mappings from atoms and their default negations to all
literals: An \emph{(extended) level mapping} $\lm$ maps every objective
literal to a natural number. Also, for any default literal $\lpnot \olit$ and
finite set of literals $\slit$, $\lm(\lpnot \olit) = \lm(\atm)$,
$\lmmin(\slit) = \min \set{\lm(\lit) | \lit \in \slit}$ and $\lmmax(\slit) =
\max \set{\lm(\lit) | \lit \in \slit}$.

\begin{definition}
	[Well-supported model of an extended program]
	Let $\prg$ be an extended program and $\lm$ a level mapping. An
	interpretation $\twi$ is a \emph{well-supported model of $\prg$ w.r.t.\
	$\lm$} if the following conditions are satisfied:
	\begin{enumerate}[1.]
		\item $\twi$ is a model of $\prg$;

		\item For every objective literal $\olit \in \twi$ there exists a rule
			$\rl \in \prg$ such that
			\[
				\hrl = \olit \land \twi \ent \brl \land \lm(\hrl) > \lmmax(\brl)
				\enspace.
			\]
	\end{enumerate}
	The set $\modws{\prg}$ of \emph{well-supported models of $\prg$} consists of
	all interpretations $\twi$ such that $\twi$ is a well-supported model of
	$\prg$ w.r.t.\ some level mapping.
\end{definition}

We obtain a generalisation of Prop.~\ref{prop:ws coincide with sm} to the
class of extended programs:

\begin{proposition}
	\label{prop:extended ws coincides with sm}
	Let $\prg$ be an extended program. Then, $\modws{\prg} = \modsm{\prg}$.
\end{proposition}
\begin{proofmain}
	See Appendix~\ref{app:proofs}, page~\pageref{proof:prop:extended ws
	coincides with sm}. \qed
\end{proofmain}


%% file: rule-updates.tex
We turn our attention to rule updates, starting with one of the most
advanced rule update semantics, the \emph{refined dynamic stable models} for
sequences of generalised programs \cite{Alferes2005}, as well as the
equivalent definition of \emph{well-supported models} \cite{Banti2005}. Then
we define the transformations for adding support for strong negation to such
semantics \cite{Alferes1996,Leite2003}.

A rule update semantics provides a way to assign stable models to a pair or
sequence of programs where each component represents an update of the
preceding ones. Formally, a \emph{dynamic logic program} (DLP) is a finite
sequence of extended programs and by $\all{\dprg}$ we denote the multiset of
all rules in the components of $\dprg$. A rule update semantics \S{} assigns
a \emph{set of \S-models}, denoted by $\mods{\dprg}$, to $\dprg$.

We focus on semantics based on the causal rejection principle
\cite{Leite1997,Alferes2000,Eiter2002,Leite2003,Alferes2005,Banti2005,Osorio2007}
which states that a rule is \emph{rejected} if it is in a direct conflict with
a more recent rule. The basic type of conflict between rules $\rla$ and
$\rlb$ occurs when their heads contain complementary literals, i.e.\ when
$\hrla = \lpnot \hrlb$. Based on such conflicts and on a stable model
candidate, a \emph{set of rejected rules} can be determined and it can be
verified that the candidate is indeed stable w.r.t.\ the remaining rules.

We define the most mature of these semantics, providing two equivalent
definitions: the \emph{refined dynamic stable models} \cite{Alferes2005}, or
\emph{\RD-semantics}, defined using a fixpoint equation, and the
\emph{well-supported models} \cite{Banti2005}, or \emph{\WS-semantics}, based
on level mappings.

\begin{definition}
	[\RD-semantics \cite{Alferes2005}]
	Let $\dprg = \seq{\prg_\lia}_{\lia < \lng}$ be a DLP without strong
	negation. Given an interpretation $\twi$, the multisets of rejected rules
	$\rejrd{\dprg, \twi}$ and of default assumptions $\defl{\dprg, \twi}$ are
	defined as follows:
	\begin{align*}
		\rejrddef,
		\\
		\defrddef.
	\end{align*}
	The set $\modrd{\dprg}$ of \emph{\RD-models of $\dprg$} consists of all
	interpretations $\twi$ such that
	\[
		\twiall = \Least{
			[ \all{\dprg} \setminus \rejrd{\dprg, \twi} ]
			\cup
			\defl{\dprg, \twi}
		}
	\]
	where $\twiall$ and $\least{\cdot}$ are defined as before.
\end{definition}

\begin{definition}
	[\WS-semantics \cite{Banti2005}]
	Let $\dprg = \seq{\prg_\lia}_{\lia < \lng}$ be a DLP without strong
	negation. Given an interpretation $\twi$ and a level mapping $\lm$, the
	multiset of rejected rules $\rejws{\dprg, \twi}$ is defined as follows:
	\begin{align*}
		\rejwsdef.
	\end{align*}
	The set $\modws{\dprg}$ of \emph{\WS-models of $\dprg$} consists of all
	interpretations $\twi$ such that for some level mapping $\lm$, the following
	conditions are satisfied:
	\begin{enumerate}[1.]
		\item $\twi$ is a model of $\all{\dprg} \setminus \rejws{\dprg, \twi}$;

		\item For every $\olit \in \twi$ there exists some rule $\rl \in
			\all{\dprg} \setminus \rejws{\dprg, \twi}$ such that
			\[
				\hrl = \olit \land \twi \ent \brl \land \lm(\hrl) > \lmmax(\brl)
				\enspace.
			\]
	\end{enumerate}
\end{definition}

Unlike most other rule update semantics, these semantics can properly deal
with tautological and other irrelevant updates, as illustrated in the
following example:

\begin{example}[Irrelevant updates]
  Consider the DLP $\dprg = \seq{\prga, \prgu}$ where programs $\prga$,
  $\prgu$ are as follows:
  \begin{align*}
    \prga: &&
    \atmst{day} &\lpif \lpnot \atmst{night}.
    &&&&&
    \atmst{stars} &\lpif \atmst{night}, \lpnot \atmst{cloudy}. 
    \\
    &&
    \atmst{night} &\lpif \lpnot \atmst{day}.
    &&&&&
    \lpnot \atmst{stars}&.
    \\
    \prgu: &&
    \atmst{stars} &\lpif \atmst{stars}.
  \end{align*}
	Note that program $\prga$ has the single stable model $\twi_1 =
	\set{\atmst{day}}$ and $\prgu$ contains a single tautological rule, i.e.\ it
	does not encode any change in the modelled domain. Thus, we expect that
	$\dprg$ also has the single stable model $\twi_1$. Nevertheless, many rule
	update semantics, such as those introduced in
	\cite{Leite1997,Alferes2000,Eiter2002,Leite2003,Sakama2003,Zhang2006,Osorio2007,Delgrande2007,Krumpelmann2012},
	are sensitive to this or other tautological updates, introducing or
	eliminating models of the original program.

	In this case, the unwanted model candidate is $\twi_2 = \set{\atmst{night},
	\atmst{stars}}$ and it is neither an \RD- nor a \WS-model of $\dprg$, though
	the reasons for this are technically different under these two semantics.
	It is not difficult to verify that, given an arbitrary level mapping $\lm$,
	the respective sets of rejected rules and the set of default assumptions are
	as follows:
  \begin{align*}
    \rejrd{\dprg, \twi_2}
    &=
    \set{
      (\atmst{stars} \lpif \atmst{night}, \lpnot \atmst{cloudy}.),
      (\lpnot \atmst{stars}.)
    },
		\\
		\rejws{\dprg, \twi_2}
		&=
		\emptyset
		,
		\\
		\defl{\dprg, \twi_2}
		&=
		\set{
			(\lpnot \atmst{cloudy}.),
			(\lpnot \atmst{day}.)
		}.
  \end{align*}
	Note that $\rejws{\dprg, \twi_2}$ is empty because, independently of $\lm$,
	no rule $\rl$ in $\prgu$ satisfies the condition $\lm(\hrl) > \lmmax(\brl)$,
	so there is no rule that could reject another rule. Thus, the atom
	$\atmst{stars}$ belongs to $\twiall[\twi_2]$ but does not belong to $\least{
		[ \all{\dprg} \setminus \rejrd{\dprg, \twi_2} ] \cup \defl{\dprg, \twi_2}
	} $, so $\twi_2$ is not an \RD-model of $\dprg$.  Furthermore, no model of
	$\all{\dprg} \setminus \rejws{\dprg, \twi_2}$ contains $\atmst{stars}$, so
	$\twi_2$ cannot be a \WS-model of $\dprg$.

	Furthermore, the resilience of \RD- and \WS-semantics is not limited to
	empty and tautological updates, but extends to other irrelevant updates as
	well \cite{Alferes2005,Banti2005}. For example, consider the DLP $\dprg' =
	\seq{\prga, \prgu'}$ where $\prgu' = \set{(\atmst{stars} \lpif
	\atmst{venus}.), (\atmst{venus} \lpif \atmst{stars}.)}$. Though the updating
	program contains non-tautological rules, it does not provide a bottom-up
	justification of any model other than $\twi_1$ and, indeed, $\twi_1$ is the
	only \RD- and \WS-model of $\dprg'$.
\end{example}

We also note that the two presented semantics for DLPs without strong negation
provide the same result regardless of the particular DLP to which they are
applied.

\begin{proposition}
	[\hspace{-.07ex}\cite{Banti2005}]
	\label{prop:ws coincides with rd}
	Let $\dprg$ be a DLP without strong negation. Then, $\modws{\dprg} =
	\modrd{\dprg}$.
\end{proposition}

In case of the stable model semantics for a single program, strong negation
can be reduced away by treating all objective literals as atoms and adding,
for each atom $\atm$, the integrity constraint $(\lpif \atm, \lnot \atm.)$ to
the program \cite{Gelfond1991}. However, this transformation does not serve
its purpose when adding support for strong negation to causal rejection
semantics for DLPs because integrity constraints have empty heads, so
according to these rule update semantics, they cannot be used to reject any
other rule. For example, a DLP such as $\seq{\set{\atm., \lnot \atm.},
\set{\atm.}}$ would remain without a stable model even though the DLP
$\seq{\set{\atm., \lpnot \atm.}, \set{\atm.}}$ does have a stable model.

To capture the conflict between opposite objective literals $\olit$ and $\lnot
\olit$ in a way that is compatible with causal rejection semantics, a slightly
modified syntactic transformation can be performed, translating such
conflicts into conflicts between objective literals and their default
negations. Two such transformations have been suggested in the literature
\cite{Alferes1996,Leite2003}, both based on the principle of coherence. For
any extended program $\prg$ and DLP $\dprg = \seq{\prg_\lia}_{\lia < \lng}$
they are defined as follows:
\begin{align*}
	\expone{\prg} 
	&=
	\prg \cup \{
		\lpnot \lnot \olit \lpif \olit.
		|
		\olit \in \olits
	\},\\
	\hspace{-0.1cm}\expone{\dprg} &= \Seq{\expone{\prg_\lia}}_{\lia < \lng},\\
	\exptwo{\prg}
	&=
	\prg \cup \{
		\lpnot \lnot \hrl \lpif \brl.
		|
		\rl \in \prg \land \hrl \in \olits
	\},\\
	\hspace{-0.1cm}\exptwo{\dprg} &= \Seq{\exptwo{\prg_\lia}}_{\lia < \lng}.
\end{align*}
These transformations lead to four possibilities for defining the semantics of
an arbitrary DLP $\dprg$: $\modrd{\expone{\dprg}}$, $\modrd{\exptwo{\dprg}}$,
$\modws{\expone{\dprg}}$ and $\modws{\exptwo{\dprg}}$. We discuss these in the
following section.


%% file: semantics-dynamic.tex
The problem with existing semantics for strong negation in rule updates is
that semantics based on the first transformation ($\expone{\dprg}$) assign too
many models to some DLPs, while semantics based on the second transformation
($\exptwo{\dprg}$) sometimes do not assign any model to a DLP that should have
one. The former is illustrated in the following example:

\begin{example}
	[Undesired side effects of the first transformation]
	\label{ex:first transformation}
	Consider the DLP $\dprg_1 = \seq{\prga, \prgu}$ where $\prga = \set{\atma.,
	\lnot \atma.}$ and $\prgu = \emptyset$. Since $\prga$ has no stable model
	and $\prgu$ does not encode any change in the represented domain, it should
	follow that $\dprg_1$ has no stable model either. However,
	$\modrd{\expone{\dprg_1}} = \modws{\expone{\dprg_1}} = \Set{\set{\atma},
	\set{\lnot \atma}}$, i.e.\ two models are assigned to $\dprg_1$ when using
	the first transformation to add support for strong negation. To verify this,
	observe that $\expone{\dprg_1} = \seq{\expone{\prga}, \expone{\prgu}}$ where
	\begin{align*}
		\expone{\prga}: &
			&
			\atm &.
			&
			\lnot \atm &.
		&&&&&
		\expone{\prgu}: &
			&
			\lpnot \atm &\lpif \lnot \atm.
		\\
			&&
			\lpnot \atm &\lpif \lnot \atm.
			&
			\lpnot \lnot \atm &\lpif \atm.
						&&&&&&&
			\lpnot \lnot \atm &\lpif \atm.
	\end{align*}
	Consider the interpretation $\twi_1 = \set{\atm}$. It is not difficult to
	verify that
	\begin{align*}
		\rejrd{\expone{\dprg_1}, \twi_1}
		&=
		\set{\lnot \atm., \lpnot \lnot \atm \lpif \atm.}
		\enspace,
		\\
		\defl{\expone{\dprg_1}, \twi_1}
		&=
		\emptyset
		\enspace,
	\end{align*}
	so it follows that
	\begin{align*}
		\Least{
			\left[
				\all{\expone{\dprg_1}} \setminus \rejrd{\expone{\dprg_1}, \twi_1}
			\right]
			\cup
			\defl{\expone{\dprg_1}, \twi_1}
		}=\\
		=\set{\atm, \lpnot \lnot \atm} 
		= \twiall[\twi_1].
	\end{align*}
	In other words, $\twi_1$ belongs to $\modrd{\expone{\dprg_1}}$ and in an
	analogous fashion it can be verified that $\twi_2 = \set{\lnot \atm}$ also
	belongs there. A similar situation occurs with $\modws{\expone{\dprg_1}}$
	since the rules that were added to the more recent program can be used to
	reject facts in the older one.
\end{example}

Thus, the problem with the first transformation is that an update by an empty
program, which does not express any change in the represented domain, may
affect the original semantics. This behaviour goes against basic and
intuitive principles underlying updates, grounded already in the classical
belief update postulates \cite{Keller1985,Katsuno1991} and satisfied by
virtually all belief update operations \cite{Herzig1999} as well as by the
vast majority of existing rule update semantics, including the original \RD-
and \WS-semantics.

This undesired behaviour can be corrected by using the second transformation
instead. The more technical reason is that it does not add any rules to a
program in the sequence unless that program already contains some original
rules. However, its use leads to another problem: sometimes \emph{no model} is
assigned when in fact a model should exist.

\begin{example}
	[Undesired side effects of the second transformation]
	\label{ex:second transformation}
	Consider again Example~\ref{ex:faulty sensor}, formalised as the DLP
	$\dprg_2 = \seq{\prga, \prgv}$ where $\prga = \set{\atm., \lnot \atm.}$ and
	$\prgv = \set{\lpnot \atm.}$. It is reasonable to expect that since $\prgv$
	resolves the conflict present in $\prga$, a stable model should be assigned
	to $\dprg_2$. However, $\modrd{\exptwo{\dprg_2}} = \modws{\exptwo{\dprg_2}}
	= \emptyset$. To verify this, observe that $\exptwo{\dprg_2} =
	\seq{\exptwo{\prga}, \exptwo{\prgv}}$ where
	\begin{align*}
		&&&&
		\exptwo{\prga}: &
			&
			\atm &.
			&
			\lnot \atm &.
		&&&&&
		\exptwo{\prgv}: &
			&
			\lpnot \atm &.
			&&&&
		\\
			&&&&&&
			\lpnot \atm &.
			&
			\lpnot \lnot \atm &.
	\end{align*}
	Given an interpretation $\twi$ and level mapping $\lm$, we conclude that
	$\rejws{\exptwo{\dprg_2}, \twi} = \Set{\atm.}$, so the facts $(\lnot \atm.)$
	and $(\lpnot \lnot \atm.)$ both belong to the program
	\[
		\all{\exptwo{\dprg_2}} \setminus \rejws{\exptwo{\dprg_2}, \twi}
		\enspace.
	\]
	Consequently, this program has no model and it follows that $\twi$ cannot
	belong to $\modws{\exptwo{\dprg_2}}$. Similarly it can be shown that
	$\modrd{\exptwo{\dprg_2}} = \emptyset$.
\end{example}

Based on this example, in the following we formulate a generic \emph{early
recovery principle} that formally identifies conditions under which
\emph{some} stable model should be assigned to a DLP. For the sake of
simplicity, we concentrate on DLPs of length 2 which are composed of facts. We
discuss a generalisation of the principle to DLPs of arbitrary length and
containing other rules than just facts in Sect.~\ref{sec:discussion}. After
introducing the principle, we define a semantics for rule updates which
directly supports both strong and default negation and satisfies the
principle.

We begin by defining, for every objective literal $\olit$, the sets of
literals $\con{\olit}$ and $\con{\lpnot \olit}$ as follows:
\begin{align*}
	& \con{\olit} = \set{\lpnot \olit, \lnot \olit}
	&& \text{and}
	&& \con{\lpnot \olit} = \set{\olit}
	\enspace.
\end{align*}
Intuitively, for every literal $\lit$, $\con{\lit}$ denotes the set of
literals that are in conflict with $\lit$.
Furthermore, given two sets of facts $\prga$ and $\prgu$, we say that
\emph{$\prgu$ solves all conflicts in $\prga$} if for each pair of rules
$\rla, \rlb \in \prga$ such that $\hrlb \in \con{\hrla}$ there is a fact $\rlc
\in \prgu$ such that either $\hrlc \in \con{\hrla}$ or $\hrlc \in
\con{\hrlb}$.

Considering a rule update semantics \S{}, the new principle simply requires
that when $\prgu$ solves all conflicts in $\prga$, \S{} will assign \emph{some
model} to $\seq{\prga, \prgu}$. Formally:
\begin{description}
	\item
		[Early recovery principle:]
		If $\prga$ is a set of facts and $\prgu$ is a consistent set of facts that
		solves all conflicts in $\prga$, then $\mods{\seq{\prga, \prgu}} \neq
		\emptyset$.
\end{description}

We conjecture that rule update semantics should generally satisfy the above
principle. In contrast with the usual behaviour of belief update operators,
the nature of existing rule update semantics ensures that recovery from
conflict is always possible, and this principle simply formalises and sharpens
the sufficient conditions for such recovery.

Our next goal is to define a semantics for rule updates that not only
satisfies the outlined principle, but also enjoys other established properties
of rule updates that have been identified over the years. Similarly as for the
original semantics for rule updates, we provide two equivalent definitions,
one based on a fixed point equation and the other one on level mappings.

To directly accommodate strong negation in the \RD-semantics, we
first need to look more closely at the set of rejected rules $\rejrd{\dprg,
\twi}$, particularly at the fact that it allows conflicting rules within the
same component of $\dprg$ to reject one another. This behaviour, along with
the constrained set of defaults $\defl{\dprg, \twi}$, is used to prevent
tautological and other irrelevant cyclic updates from affecting the semantics.
However, in the presence of strong negation, rejecting conflicting rules
within the same program has undesired side effects. For example, the early
recovery principle requires that some model be assigned to the DLP
$\seq{\set{\atm., \lnot \atm.}, \set{\lpnot \atm}}$ from
Example~\ref{ex:second transformation}, but if the rules in the initial
program reject each other, then the only possible stable model to assign is
$\emptyset$. However, such a stable model would violate the causal rejection
principle since it does not satisfy the initial rule $(\lnot \atm.)$ and there
is no rule in the updating program that overrides it.

To overcome the limitations of this approach to the prevention of tautological
updates, we disentangle rule rejection per se from ensuring that rejection is
done without cyclic justifications. We introduce the set of rejected rules
$\rejrds{\dprg, \slit}$ which directly supports strong negation and does not
allow for rejection within the same program. Prevention of cyclic rejections
is done separately by using a customised immediate consequence operator
$\imconrds$. Given a stable model candidate $\twi$, instead of verifying that
$\twiall$ is the least fixed point of the usual consequence operator, as done
in the \RD-semantics using $\least{\cdot}$, we verify that $\twiall$ is the
least fixed point of $\imconrds$.

\begin{definition}
	[Extended \RD-semantics]
	Let $\dprg = \seq{\prg_\lia}_{\lia < \lng}$ be a DLP. Given an
	interpretation $\twi$ and a set of literals $\slit$, the multiset of
	rejected rules $\rejrds{\dprg, \slit}$, the remainder $\rem{\dprg, \slit}$
	and the consequence operator $\imconrds$ are defined as follows:
	\begin{align*}
		\rejrdsdef,
		\\
		\rem{\dprg, \slit}
		&=
		\all{\dprg} \setminus \rejrds{\dprg, \slit}
		\enspace,
		\\
		\imconrds(\slit)
		&=
		\bigl\{\,
			\hrla
			\mid
			\rla \in \left(
				\rem{\dprg, \twiall}
				\cup
				\defl{\twi}
			\right)
						\land
			\brla \subseteq \slit
		\\ & \hspace{0cm}
			\land
			\lnot \left(
				\exists \rlb \in \rem{\dprg, \slit}
				:
				\hrlb \in \con{\hrla}
				\land
				\brlb \subseteq \twiall
			\right)
		\,\bigr\}.
	\end{align*}
	Furthermore, $\imconrds^0(\slit) = \slit$ and for every $\lic \geq 0$,
	$\imconrds^{\lic + 1}(\slit) = \imconrds(\imconrds^\lic(\slit))$. The set
	$\modrds{\dprg}$ of \emph{extended \RD-models of $\dprg$} consists of all
	interpretations $\twi$ such that
	\[
		\twiall = \bigcup_{\lic \geq 0} \imconrds^\lic(\emptyset)
		\enspace.
	\]
\end{definition}

Adding support for strong negation to the \WS-semantics is done by modifying
the set of rejected rules $\rejws{\dprg, \twi}$ to account for the new type of
conflict. Additionally, in order to ensure that rejection of a literal $\lit$
cannot be based on the assumption that some conflicting literal $\lit' \in
\con{\lit}$ is true, a rejecting rule $\rlb$ must satisfy the stronger
condition $\lmmin(\con{\lit}) > \lmmax(\brlb)$. Finally, to prevent defeated
rules from affecting the resulting models, we require that all supporting
rules belong to $\rem{\dprg, \twiall}$.

\begin{definition}
	[Extended \WS-semantics]
	Let $\dprg = \seq{\prg_\lia}_{\lia < \lng}$ be a DLP. Given an
	interpretation $\twi$ and a level mapping $\lm$, the multiset of rejected
	rules $\rejwss{\dprg, \twi}$ is defined by:
	\begin{align*}
		\rejwssdef.
	\end{align*}
	The set $\modwss{\dprg}$ of \emph{extended \WS-models of $\dprg$} consists
	of all interpretations $\twi$ such that for some level mapping $\lm$, the
	following conditions are satisfied:
	\begin{enumerate}[1.]
		\item $\twi$ is a model of $\all{\dprg} \setminus \rejwss{\dprg, \twi}$;

		\item For every $\olit \in \twi$ there exists some rule $\rl \in
			\rem{\dprg, \twiall}$ such that
			\[
				\hrl = \olit \land \twi \ent \brl \land \lm(\hrl) > \lmmax(\brl)
				\enspace.
			\]
	\end{enumerate}
\end{definition}

The following theorem establishes that the two defined semantics are
equivalent:

\begin{theorem}
	\label{thm:extended ws coincides with extended rd}
	Let $\dprg$ be a DLP. Then, $\modwss{\dprg} = \modrds{\dprg}$.
\end{theorem}
\begin{proofmain}
	See Appendix~\ref{app:proofs}, page~\pageref{proof:thm:extended ws coincides
	with extended rd}. \qed
\end{proofmain}

Also, on DLPs without strong negation they coincide with the original
semantics.

\begin{theorem}
	\label{thm:extended semantics coincide with regular}
	Let $\dprg$ be a DLP without strong negation. Then,
$\modwss{\dprg} = \modrds{\dprg} = \modws{\dprg} = \modrd{\dprg}$.
\end{theorem}
\begin{proofmain}
	See Appendix~\ref{app:proofs}, page~\pageref{proof:thm:extended semantics
	coincide with regural}. \qed
\end{proofmain}

Furthermore, unlike the transformational semantics for strong negation, the
new semantics satisfy the early recovery principle.

\begin{theorem}
	\label{thm:extended ws early recovery}
	The extended \RD-semantics and extended \WS-semantics satisfy the early
	recovery principle.
\end{theorem}
\begin{proofmain}
	See Appendix~\ref{app:proofs}, page~\pageref{proof:thm:extended ws
	early recovery}. \qed
\end{proofmain}






%% file: properties.tex
In this section we take a closer look at the formal and computational
properties of the proposed rule update semantics.

The various approaches to rule updates
\cite{Leite1997,Alferes2000,Eiter2002,Leite2003,Sakama2003,Alferes2005,Banti2005,Zhang2006,Sefranek2006,Osorio2007,Delgrande2007,Sefranek2011,Krumpelmann2012}
share a number of basic characteristics. For example, all of them generalise
stable models, i.e., the models they assign to a sequence $\seq{\prg}$ (of
length 1) are exactly the stable models of $\prg$. Similarly, they adhere to
the principle of primacy of new information \cite{Dalal1988}, so models
assigned to $\seq{\prg_\lia}_{\lia < \lng}$ satisfy the latest program
$\prg_{\lng - 1}$. However, they also differ significantly in their technical
realisation and classes of supported inputs, and desirable properties such as
immunity to tautologies are violated by many of them.

\begin{table*}[t]
	\caption{Desirable properties of rule update semantics}
	\label{tab:rule update properties}
	\begin{tabular}
		{p{13.5em}p{34em}}
		\toprule
		\textbf{Generalisation of stable models}
		&
		$\mods{\seq{\prg}} = \modsm{\prg}$.
		\\ \midrule
		\textbf{Primacy of new information}
		&
		If $\twi \in \mods{\seq{\prg_\lia}_{\lia < \lng}}$, then $\twi \ent
		\prg_{\lng - 1}$.
		\\ \midrule
		\textbf{Fact update}
		&
		A sequence of consistent sets of facts $\seq{\prg_\lia}_{\lia < \lng}$ has
		the single model
		$
			\Set{
				\olit \in \olits |
				\exists \lia < \lng : (\olit.) \in \prg_\lia
				\land
				(\forall \lib > \lia :
					\Set{\lopp{\olit}., \lcmp{\olit}.} \cap \prg_\lib = \emptyset
				)
			}
		$.
		\\ \midrule
		\textbf{Support}
		&
		If $\twi \in \mods{\dprg}$ and $\olit \in \twi$, then there is some rule
		$\rl \in \all{\dprg}$ such that $\hrl = \olit$ and $\twi \ent \brl$.
		\\ \midrule
		\textbf{Idempotence}
		&
		$\mods{\seq{\prga, \prga}} = \mods{\seq{\prga}}$.
		\\ \midrule
		\textbf{Absorption}
		&
		$\mods{\seq{\prga, \prgu, \prgu}} = \mods{\seq{\prga, \prgu}}$.
		\\ \midrule
		\textbf{Augmentation}
		&
		If $\prgu \subseteq \prgv$, then $\mods{\seq{\prga, \prgu, \prgv}} =
		\mods{\seq{\prga, \prgv}}$.
		\\ \midrule
		\textbf{Non-interference}
		&
		If $\prgu$ and $\prgv$ are over disjoint alphabets, 
		then $\mods{\seq{\prga, \prgu, \prgv}} = \mods{\seq{\prga, \prgv,
		\prgu}}$.
		\\ \midrule
		\vspace{-.5ex}\textbf{Immunity to empty updates}
		&
		\vspace{-1.5ex}
		If $\prg_\lib = \emptyset$, then $\mods{\seq{\prg_\lia}_{\lia < \lng}} =
		\Mods{\Seq{\prg_\lia}_{\lia < \lng \land \lia \neq \lib}}$.
		\\ \midrule
		\textbf{Immunity to tautologies}
		&
		If $\seq{\prgb_\lia}_{\lia < \lng}$ is a sequence of sets of tautologies,
		then $\mods{\seq{\prga_\lia \cup \prgb_\lia}_{\lia < \lng}} =
		\mods{\seq{\prga_\lia}_{\lia < \lng}}$.
		\\ \midrule
		\textbf{Causal rejection principle}
		&
		For every $\lia < \lng$, $\rl \in \prg_\lia$ and $\twi \in
		\mods{\seq{\prg_\lia}_{\lia < \lng}}$, if $\twi \nent \rl$, then there
		exists some $\rlb \in \prg_\lib$ with $\lib > \lia$ such that $\hrlb \in
		\con{\hrla}$ and $\twi \ent \brlb$.
		\\ \bottomrule
	\end{tabular}
\end{table*}

Table~\ref{tab:rule update properties} lists many of the generic properties
proposed for rule updates that have been identified and formalised throughout
the years \cite{Leite1997,Eiter2002,Leite2003,Alferes2005}. The rule update
semantics we defined in the previous section enjoys all of them.

\begin{theorem}
	\label{thm:extended ws other properties}
	The extended \RD-semantics and extended \WS-semantics satisfy all properties
	listed in Table~\ref{tab:rule update properties}.
\end{theorem}
\begin{proofmain}
	See Appendix~\ref{app:proofs}, page~\pageref{proof:thm:extended ws
	other properties}. \qed
\end{proofmain}


Our semantics also retains the same computational complexity as the stable
models.

\begin{theorem}
	\label{thm:complexity}
	Let $\dprg$ be a DLP. The problem of deciding whether some $\twi \in
	\modwss{\dprg}$ exists is \NP-complete. Given a literal $\lit$, the problem
	of deciding whether for all $\twi \in \modwss{\dprg}$ it holds that $\twi
	\ent \lit$ is \coNP-complete.
\end{theorem}
\begin{proofmain}
	See Appendix~\ref{app:proofs}, page~\pageref{proof:thm:complexity}. \qed
\end{proofmain}


%% file: discussion.tex
In this paper we have identified shortcomings in the existing semantics for
rule updates that  fully support both strong and default negation, and
proposed a generic \emph{early recovery principle} that captures them
formally. Subsequently, we provided two equivalent definitions of a new
semantics for rule updates.

We have shown that the newly introduced rule update semantics constitutes a
strict improvement upon the state of the art in rule updates as it enjoys the
following combination of characteristics, unmatched by any previously existing
semantics:
\begin{textitem}
	\item It allows for both strong and default negation in heads of rules,
		making it possible to move between any pair of epistemic states by means
		of updates;

	\item It satisfies the \emph{early recovery principle} which guarantees the
		existence of a model whenever all conflicts in the original program are
		satisfied;

	\item It enjoys all rule update principles and desirable properties reported
		in Table~\ref{tab:rule update properties};

	\item It does not increase the computational complexity of the stable model
		semantics upon which it is based.
\end{textitem}

However, the early recovery principle, as it is formulated in
Sect.~\ref{sec:semantics dynamic}, only covers a single update of a set of
facts by another set of facts. Can it be generalised further without rendering
it too strong? Certain caution is appropriate here, since in general the
absence of a stable model can be caused by odd cycles or simply by the
fundamental differences between different approaches to rule update, and the
purpose of this principle is not to choose which approach to take.

Nevertheless, one generalisation that should cause no harm is the
generalisation to iterated updates, i.e.\ to sequences of sets of facts.
Another generalisation that appears very reasonable is the generalisation to
\emph{acyclic DLPs}, i.e.\ DLPs such that $\all{\dprg}$ is an acyclic program.
An acyclic program has at most one stable model, and if we guarantee that all
potential conflicts within it certainly get resolved, we can safely conclude
that the rule update semantics should assign some model to it. We formalise
these ideas in what follows.

We say that a program $\prg$ is \emph{acyclic} \cite{Apt1991} if for some
level mapping $\lm$, such that for every $\olit \in \olits$, $\lm(\olit) =
\lm(\lnot \olit)$, and every rule $\rl \in \prg$ it holds that $\lm(\hrl) >
\lmmax(\brl)$. Given a DLP $\dprg = \seq{\prg_\lia}_{\lia < \lng}$, we say
that \emph{all conflicts in $\dprg$ are solved} if for every $\lia < \lng$ and
each pair of rules $\rla, \rlb \in \prg_\lia$ such that $\hrlb \in
\con{\hrla}$ there is some $\lib > \lia$ and a fact $\rlc \in \prg_\lib$ such
that either $\hrlc \in \con{\hrla}$ or $\hrlc \in \con{\hrlb}$.

\begin{description}
	\item
		[Generalised early recovery principle:]
		If $\all{\dprg}$ is acyclic and all conflicts in $\dprg$ are solved, then
		$\mods{\dprg} \neq \emptyset$.
\end{description}

Note that this generalisation of the early recovery principle applies to a
much broader class of DLPs than the original one. We illustrate this in the
following example:

\begin{example}[Recovery in a stratified program]
	Consider the following programs programs $\prga$, $\prgu$ and $\prgv$:
  \begin{align*}
	\prga:
		&&\atma 		&\lpif \atmb, \lpnot \atmc.	&&&\lpnot \atma 	&\lpif \atmd.	&&\atmb.			&&\atmd \lpif \atmb.\\
	\prgu:
		&&\lnot \atma	&.					&&& 				&			&&\atmc \lpif \atmb.	&\lnot &\atmc \lpif \atmb, \atmd.\\
	\prgv:
		&&			&					&&&				&			&\lpnot& \atmc.
	\end{align*}
	Looking more closely at program $\prga$, we see that atoms $\atmb$ and
	$\atmd$ are derived by the latter two rules inside it while atom $\atmc$ is
	false by default since there is no rule that could be used to derive its
	truth. Consequently, the bodies of the first two rules are both satisfied
	and as their heads are conflicting, $\prga$ has no stable model.  The single
	conflict in $\prga$ is solved after it is updated by $\prgu$, but then
	another conflict is introduced due to the latter two rules in the updating
	program. This second conflict can be solved after another update by $\prgv$.
	Consequently, we expect that some stable model be assigned to the DLP
	$\seq{\prga, \prgu, \prgv}$.

	The original early recovery principle does not impose this because the DLP
	in question has more than two components and the rules within it are not
	only facts. However, the DLP is acyclic, as shown by any level mapping
	$\lm$ with $\lm(\atma) = 3$, $\lm(\atmb) = 0$, $\lm(\atmc) = 2$ and
	$\lm(\atmd) = 1$, so the generalised early recovery principle does apply.
	Furthermore, we also find the single extended \RD-model of $\seq{\prga,
	\prgu, \prgv}$ is $\set{\lnot \atma, \atmb, \lnot \atmc, \atmd}$, i.e.\ the
	semantics respects the stronger principle in this case.
\end{example}

Moreover, as established in the following theorem, it is no coincidence that
the extended \RD-semantics respects the stronger principle in the above
example -- the principle is generally satisfied by the semantics introduced in
this paper.

\begin{theorem}
	\label{thm:generalised early recovery}
	The extended \RD-semantics and extended \WS-semantics satisfy the
	generalised early recovery principle.
\end{theorem}
\begin{proofmain}
	See Appendix~\ref{app:proofs}, page~\pageref{proof:thm:generalised early
	recovery}. \qed
\end{proofmain}

Both the original and the generalised early recovery principle can guide the
future addition of full support for both kinds of negations in other
approaches to rule updates, such as those proposed in
\cite{Sakama2003,Zhang2006,Delgrande2007,Krumpelmann2012}, making it possible
to reach any belief state by updating the current program. Furthermore, adding
support for strong negation is also interesting in the context of recent
results on program revision and updates that are performed on the
\emph{semantic level}, ensuring syntax-independence of the respective methods
\cite{Delgrande2013,Slota2013a,Slota2012,SlotaL10}, in the context of finding 
suitable condensing operators \cite{SlotaL13}, and unifying with updates in classical logic \cite{SlotaL12JELIA}.




%% file: proofs.tex
%
\begin{definition}
	[Immediate consequence operator]
	Let $\prg$ be an extended program. We define the \emph{immediate consequence
	operator $\imcon$} for every interpretation $\twi$ as follows:
	\[
		\imcon(\twi) = \Set{
			\hrl
			|
			\rl \in \prg
			\land
			\brl \subseteq \twi
		}
		\enspace.
	\]
	Furthermore, $\imcon^0(\twi) = \twi$ and $\imcon^{\lic + 1}(\twi) =
	\imcon(\imcon^\lic(\twi))$ for every $\lic \geq 0$.
\end{definition}

\begin{lemma}
	\label{lemma:imcon}
	Let $\prg$ be an extended program. Then $\bigcup_{\lic \geq 0}
	\imcon^\lic(\emptyset)$ is the least fixed point of $\imcon$ and coincides
	with $\least{\prg}$.
\end{lemma}
\begin{proof}
	Recall that $\least{\cdot}$ denotes the least model of the argument program
	in which all literals are treated as propositional atoms. It follows from
	Kleene's fixed point theorem that $\slit = \bigcup_{\lic \geq 0}
	\imcon^\lic(\emptyset)$ is the least fixed point of $\imcon$. To verify that
	$\slit$ is a model of $\prg$, take some rule $\rl \in \prg$ such that $\brl
	\subseteq \slit$. By the definition of $\imcon$, $\hrl \in \imcon(\slit) =
	\slit$. Also, for any model $\slit'$ of $\prg$ it follows that $\emptyset
	\subseteq \slit'$ and whenever $\slit'' \subseteq \slit'$, also
	$\imcon(\slit'') \subseteq \slit'$. Thus, for all $\lic \geq 0$,
	$\imcon^\lic(\emptyset) \subseteq \slit'$, implying that $\slit \subseteq
	\slit'$. In other words, $\slit$ is the least model of $\prg$ when all
	literals are treated as propositional atoms.
\end{proof}

\begin{proposition*}
	{prop:extended ws coincides with sm}
	Let $\prga$ be an extended program. Then, $\modws{\prga} = \modsm{\prga}$.
\end{proposition*}
\begin{proof}
	\label{proof:prop:extended ws coincides with sm}
	First suppose that $\twi$ belongs to $\modws{\prga}$. It follows that $\twi
	\ent \prga$ and there exists a level mapping $\lm$ such that for every
	objective literal $\olit \in \twi$ there is a rule $\rl \in \prga$ such that
	$\hrl = \olit$, $\twi \ent \brl$ and $\lm(\hrl) > \lmmax(\brl)$. We need to
	prove that
	\[
		\twiall = \least{
			\prga
			\cup
			\set{
				\lpnot \olit.
				|
				\olit \in \olits \setminus \twi
			}
		}
		\enspace.
	\]
	Put $\prgb = \prga \cup \set{\lpnot \olit. | \olit \in \olits \setminus
	\twi}$. By Lemma~\ref{lemma:imcon}, it suffices to prove that 
	\[
		\twiall = \bigcup_{\lic \geq 0} \imcon[\prgb]^\lic(\emptyset)
		\enspace.
	\]
	Let $\slit = \bigcup_{\lic \geq 0} \imcon[\prgb]^\lic(\emptyset)$ and take
	some $\lit \in \twiall$. If $\lit$ is a default literal $\lpnot \olit$, then
	clearly $\lit$ belongs to $\imcon[\prgb](\emptyset) \subseteq \slit$. In the
	principal case, $\lit$ is an objective literal $\olit$, so there exists a
	rule $\rl \in \prg$ such that $\hrl = \olit$, $\twi \ent \brl$ and
	$\lm(\hrl) > \lmmax(\brl)$. We proceed by induction on $\lm(\olit)$:
	\begin{enumerate}[1$^\circ$]
		\item If $\lm(\olit) = 0$, then we arrive at a conflict: $0 = \lm(\olit)
			= \lm(\hrl) > \lmmax(\brl) \geq 0$.

		\item If $\lm(\olit) = \lic + 1$, then, since $\twi \ent \brl$ and
			$\lm(\olit) > \lmmax(\brl)$, from the inductive assumption we obtain
			that $\brl \subseteq \slit$. Thus, since $\slit$ is a fixed point of
			$\imcon[\prgb]$, we conclude that $\slit$ contains $\olit$.
	\end{enumerate}
	For the converse inclusion, we prove by induction on $\lic$ that
	$\imcon[\prgb]^\lic(\emptyset)$ is a subset of $\twiall$:
	\begin{enumerate}[1$^\circ$]
		\item For $\lic = 0$ the claim trivially follows from the fact that
			$\imcon[\prgb]^0(\emptyset) = \emptyset$.

		\item Suppose that $\lit$ belongs to $\imcon[\prgb]^{\lic +
			1}(\emptyset)$. It follows that for some rule $\rl \in \prgb$, $\hrl =
			\lit$ and $\brl \subseteq \imcon[\prgb]^\lic(\emptyset)$. From the
			inductive assumption we obtain that $\imcon[\prgb]^\lic(\emptyset)$ is a
			subset of $\twiall$, so $\twi \ent \brl$. Consequently, since $\twi$ is
			a model of $\prg$ (and thus of $\prgb$ as well), $\twi \ent \lit$.
			Equivalently, $\lit \in \twiall$.
	\end{enumerate}

	Now suppose that $\twi \in \modsm{\prga}$. It easily follows that $\twi$ is
	a model of $\prga$. Furthermore,
	\[
		\twiall = \least{
			\prga
			\cup
			\set{
				\lpnot \olit.
				|
				\olit \in \olits \setminus \twi
			}
		}
		\enspace.
	\]
	Put $\prgb = \prga \cup \set{\lpnot \olit. | \olit \in \olits \setminus
	\twi}$. By Lemma~\ref{lemma:imcon},
	\[
		\twiall = \bigcup_{\lic \geq 0} \imcon[\prgb]^\lic(\emptyset)
		\enspace.
	\]
	Let $\lm$ be a level mapping defined for any objective
	literal $\olit \in \twi$ as follows:
	\[
		\lm(\olit)
		=
		\min \set{
			\lic
			|
			\lic \geq 0
			\land
			\olit \in \imcon[\prgb]^\lic(\emptyset)
		}
		\enspace.
	\]
	Also, for every $\olit \in \olits \setminus \twi$, $\lm(\olit) = 0$. We
	need to prove that for every objective literal $\olit \in \twi$ there exists
	a rule $\rl \in \prg$ such that $\hrl = \olit$, $\twi \ent \brl$ and
	$\lm(\hrl) > \lmmax(\brl)$. By the definition of $\lm$, there is no literal
	$\olit \in \twi$ with $\lm(\olit) = 0$, so suppose that $\lm(\olit) = \lic
	+ 1$ for some $\lic \geq 0$. Then there is some rule $\rl \in \prgb$ such
	that $\hrl = \olit$ and $\brl \subseteq \imcon[\prgb]^\lic(\emptyset)$. It
	immediately follows that $\rl$ belongs to $\prga$, $\twi \ent \brl$ and
	$\lmmax(\brl) \leq \lic < \lic + 1 = \lm(\olit)$.
\end{proof}

\begin{theorem*}
	{thm:extended ws early recovery}
	The extended \RD-semantics and extended \WS-semantics satisfy the early
	recovery principle.
\end{theorem*}
\begin{proof}
	\label{proof:thm:extended ws early recovery}
	Suppose that $\prga$ is a set of facts and $\prgu$ is a consistent set of
	facts that solves all conflicts in $\prga$ and put
	\[
		\twi = \Set{
			\olit \in \olits
			|
			(\olit.) \in \prga \cup \prgu
			\land
			\Set{\lopp{\olit}., \lcmp{\olit}.} \cap \prgu = \emptyset
		}
		\enspace.
	\]
	Our goal is to show that $\twi$ belongs to $\modwss{\seq{\prga, \prgu}}$.

	First we verify that $\twi$ is a consistent set of objective literals, i.e.\
	that it is an interpretation. Suppose that for some $\olit \in \olits$, both
	$\olit$ and $\lnot \olit$ belong to $\twi$. It follows that both $(\olit.)$
	and $(\lnot \olit.)$ belong to $\prga \cup \prgu$ and at the same time
	neither of them belongs to $\prgu$. Thus, both must belong to $\prga$ and we
	obtain a conflict with the assumption that $\prgu$ solves all conflicts in
	$\prga$.
	
	Now consider a level mapping $\lm$ such that $\lm(\olit) = 1$ for all
	$\olit \in \olits$. We will show that $\twib$ is an extended \WS-model of
	$\dprg$ w.r.t.\ $\lm$. Note that
	\begin{align*}
		\rejwss{\seq{\prga, \prgu}, \twi}
		&=
		\{
			\rla \in \prga
			|
			\exists \rlb \in \prgu
			:
			\hrlb \in \con{\hrla}
			\land
			\twi \ent \brlb\\
			& \hspace{3.0cm}\land
			\lm(\hrlb) > \lmmax(\brlb)
		\}
		\\
		&=
		\Set{
			\rla \in \prga
			|
			\exists \rlb \in \prgu
			:
			\hrlb \in \con{\hrla}
		}
	\end{align*}
	In order to prove that $\twi$ is a model of $\all{\seq{\prga, \prgu}}
	\setminus \rejwss{\seq{\prga, \prgu}, \twi}$, take some rule
	\[
		(\lit.)
		\in
		\all{\seq{\prga, \prgu}}
		\setminus
		\rejwss{\seq{\prga, \prgu}, \twi}
		\enspace.
	\]
	We consider four cases:
	\begin{enumerate}[a)]
		\item If $\lit$ is an objective literal $\olit$ and $(\olit.)$ belongs to
			$\prga$, then it follows from the definition of $\twi$ and the
			definition of $\rejwss{\seq{\prga, \prgu}, \twi}$ that $\olit \in \twi$,
			Thus, $\twi \ent \lit$.

		\item If $\lit$ is an objective literal $\olit$ and $(\olit.)$ belongs to
			$\prgu$, then it follows from the definition of $\twi$ and the
			assumption that $\prgu$ is consistent that $\olit \in \twi$. Thus, $\twi
			\ent \lit$.

		\item If $\lit$ is a default literal $\lpnot \olit$ and $(\lpnot \olit.)$
			belongs to $\prga$, then it follows from the definition of $\twi$,
			definition of $\rejwss{\seq{\prga, \prgu}, \twi}$ and the assumption
			that $\prgu$ solves all conflicts in $\prga$ that $\olit \notin \twi$.
			Thus, $\twi \ent \lit$.

		\item If $\lit$ is a default literal $\lpnot \olit$ and $(\lpnot \olit.)$
			belongs to $\prgu$, then it follows from the definition of $\twi$ that
			$\olit \notin \twi$. Thus, $\twi \ent \lit$.
	\end{enumerate}
	
	Finally, we need to demonstrate that for every $\olit \in \twi$ there exists
	some rule $\rl \in \all{\seq{\prga, \prgu}} \setminus \rejwss{\seq{\prga,
	\prgu}, \twi}$ such that $\hrl = \olit$, $\twi \ent \brl$ and $\lm(\hrl) >
	\lmmax(\brl)$. This follows immediately from the definition of $\twi$ and of
	$\rejwss{\seq{\prga, \prgu}, \twi}$.
\end{proof}

\begin{lemma}
	\label{lemma:extended ws is extended rd}
	Let $\dprg$ be a DLP. Then, $\modwss{\dprg} \subseteq \modrds{\dprg}$.
\end{lemma}
\begin{proof}
	Let $\dprg = \seq{\prg_\lia}_{\lia < \lng}$ be a DLP and suppose that $\twi$
	belongs to $\modwss{\dprg}$. For every $\lic \geq 0$, put
	\[
		\twi_\lic = \imconrds^\lic(\emptyset)
		\enspace.
	\]
	We need to prove that $\twiall = \bigcup_{\lic \geq 0} \twi_\lic$.

	To show that $\twiall$ is a subset of $\bigcup_{\lic \geq 0} \twi_\lic$,
	consider some literal $\lit \in \twiall$ and let $\lm(\lit) = \lic$. We
	prove by induction on $\lic$ that $\lit$ belongs to $\twi_{\lic + 1}$:
	\begin{enumerate}[1$^\circ$]
		\item If $\lic = 0$, then it follows from the assumption that $\twi$ is an
			extended \WS-model of $\dprg$ that $\lit$ must be a default literal
			since if it were an objective literal, there would have exist a rule
			$\rl$ with $\hrl = \lit$ and $\lm(\hrl) > \lmmax(\brl)$, which is
			impossible since $\lmmax(\brl) \geq 0$. Thus, $\lit$ is a default
			literal $\lpnot \olit$ and we obtain $(\lpnot \olit.) \in \defl{\twi}$.
			Recall that
			\begin{align*}
				\twi_1 &=
				\imconrds(\emptyset)=\\
				&=\bigl\{\,
					\hrla
					\mid
					\rla \in \left(
						\rem{\dprg, \twiall}
						\cup
						\defl{\twi}
					\right)
					\land
					\brla \subseteq \emptyset
				\\ 
					&\land
					\lnot \left(
						\exists \rlb \in \rem{\dprg, \emptyset}
						:
						\hrlb \in \con{\hrla}
						\land
						\brlb \subseteq \twiall
					\right)
				\,\bigr\}.
			\end{align*}
			Thus, to prove that $\lit$ belongs to $\twi_1$, it remains to verify
			that
			\[
				\lnot \left(
					\exists \rlb \in \rem{\dprg, \emptyset}
					:
					\hrlb = \olit
					\land
					\brlb \subseteq \twiall
				\right)
				\enspace.
			\]
			Take some $\lia < \lng$ and some rule $\rlb \in \prg_\lia$ such that
			$\hrlb = \olit$ and $\brlb \subseteq \twiall$. It follows from the
			assumption that $\twi$ is a model of $\all{\dprg} \setminus
			\rejwss{\dprg, \twi}$ that $\rlb$ belongs to $\rejwss{\dprg, \twi}$. In
			other words,
			\[
				\exists \lib > \lia \; \exists \rlb' \in \prg_\lib
				:
				\hrl[\rlb'] \in \con{\hrlb}
				\land
				\twi \ent \brl[\rlb']
				\land
				\lmmin\!\left(\con{\hrlb}\right) > \lmmax(\brl[\rlb'])
				\enspace.
			\]
			Since $\lpnot \olit$ belongs to $\con{\hrlb}$, we obtain that
			$\lmmax(\brl[\rlb']) < 0$, which is not possible. Thus, no such $\rlb'$
			may exist and we conclude that no $\rlb$ exists either, as desired.

		\item Suppose that the claim holds for all $\lic' < \lic$, we prove it for
			$\lic$. Note that
			\begin{align*}
				\twi_{\lic + 1}
				&=\imconrds(\twi_\lic)=\\
				&=
				\bigl\{\,
					\hrla
					\mid
					\rla \in \left(
						\rem{\dprg, \twiall}
						\cup
						\defl{\twi}
					\right)
					\land
					\brla \subseteq \twi_\lic
				\\ & \hspace{0.0cm}
					\land
					\lnot \left(
						\exists \rlb \in \rem{\dprg, \twi_\lic}
						:
						\hrlb \in \con{\hrla}
						\land
						\brlb \subseteq \twiall
					\right)
				\,\bigr\}.
			\end{align*}
			To show that for some rule $\rla \in (\rem{\dprg, \twiall} \cup
			\defl{\twi})$, $\hrla = \lit$ and $\brla \subseteq \twi_\lic$, we
			consider two cases:
			\begin{enumerate}[a)]
				\item If $\lit$ is an objective literal $\olit$, then it follows from
					the assumption that $\twi$ belongs to $\modwss{\dprg}$ that there
					exists some some rule $\rla \in \rem{\dprg, \twiall}$ such that
					$\hrla = \olit$, $\twi \ent \brla$ and $\lm(\hrla) > \lmmax(\brla)$.
					Furthermore, it follows by the inductive assumption that $\brla
					\subseteq \twi_\lic$.

				\item If $\lit$ is a default literal $\lpnot \olit$, then it
					immediately follows that $\rla = (\lpnot \olit.)$ belongs to
					$\defl{\twi}$.
			\end{enumerate}
			It remains to verify that
			\[
				\lnot \left(
					\exists \rlb \in \rem{\dprg, \twi_\lic}
					:
					\hrlb \in \con{\hrla}
					\land
					\brlb \subseteq \twiall
				\right)
				\enspace.
			\]
			Take some $\lia < \lng$ and some rule $\rlb \in \prg_\lia$ such that
			$\hrlb \in \con{\hrla}$ and $\brlb \subseteq \twiall$. It follows from the
			assumption that $\twi$ is a model of $\all{\dprg} \setminus
			\rejwss{\dprg, \twi}$ that $\rlb$ belongs to $\rejwss{\dprg, \twi}$. In
			other words,
			\[
				\exists \lib > \lia \; \exists \rlb' \in \prg_\lib
				:
				\hrl[\rlb'] \in \con{\hrlb}
				\land
				\twi \ent \brl[\rlb']
				\land
				\lmmin\!\left(\con{\hrlb}\right) > \lmmax(\brl[\rlb'])
				\enspace.
			\]
			Since $\hrla \in \con{\hrlb}$, it follows that $\lmmax(\brl[\rlb']) <
			\lm(\hrla) = \lic$ and from the inductive assumption we obtain that
			$\brl[\rlb'] \subseteq \twi_\lic$. Thus, it follows that $\rlb$ belongs
			to $\rejrds{\dprg, \twi_\lic}$, as we needed to show.
	\end{enumerate}

	For the converse inclusion, suppose that $\lit \in \twi_\lic$ for some $\lic
	\geq 0$. We prove by induction on $\lic$ that $\lit$ belongs to $\twiall$.
	\begin{enumerate}[1$^\circ$]
		\item For $\lic = 0$ the claim trivially follows since $\twi_0 =
			\emptyset$.
			
		\item Assume that the claim holds for $\lic$, we prove it $\lic + 1$.
			Recall that
			\begin{align*}
				\twi_{\lic + 1}
				&=\imconrds(\twi_\lic)=\\
				&=
				\bigl\{\,
					\hrla
					\mid
					\rla \in \left(
						\rem{\dprg, \twiall}
						\cup
						\defl{\twi}
					\right)
					\land
					\brla \subseteq \twi_\lic
				\\ & \hspace{0.0cm}
					\land
					\lnot \left(
						\exists \rlb \in \rem{\dprg, \twi_\lic}
						:
						\hrlb \in \con{\hrla}
						\land
						\brlb \subseteq \twiall
					\right)
				\,\bigr\}.
			\end{align*}
			Thus, if $\lit$ belongs to $\twi_{\lic + 1}$, then one of the following
			cases occurs:
			\begin{enumerate}[a)]
				\item If $\lit = \hrl$ for some $\rl \in \rem{\dprg, \twiall}$ such
					that $\brl \subseteq \twi_\lic$, then by the inductive assumption we
					obtain $\twi \ent \brl$ and since $\rejrds{\dprg, \twiall}$ is a
					superset of $\rejwss{\dprg, \twi}$, it follows that $\rl$ belongs to
					$\all{\dprg} \setminus \rejwss{\dprg, \twi}$. Consequently, since
					$\twi$ is a model of $\all{\dprg} \setminus \rejwss{\dprg, \twi}$,
					it follows that $\lit \in \twiall$.

				\item If $\lit = \hrl$ for some $\rl \in \defl{\twi}$, then it
					immediately follows that $\lit \in \twiall$.
			\end{enumerate}
	\end{enumerate}
\end{proof}

\begin{lemma}
	\label{lemma:extended rd is extended ws}
	Let $\dprg$ be a DLP. Then, $\modrds{\dprg} \subseteq \modwss{\dprg}$.
\end{lemma}
\begin{proof}
	Let $\dprg = \seq{\prg_\lia}_{\lia < \lng}$ be a DLP and suppose that $\twi$
	belongs to $\modrds{\dprg}$. Let the level mapping $\lm$ be defined for
	objective literal $\olit$ as follows:
	\[
		\lm(\olit) = \min \Set{
			\lic \geq 0
			|
			\imconrds^\lic(\emptyset)
			\cap
			\set{\olit, \lpnot \olit}
			\neq
			\emptyset
		}
		\enspace.
	\]
	Note that $\lm(\olit)$ is well-defined since $\twiall \cap \set{\olit,
	\lpnot \olit} \neq \emptyset$ and, by our assumption, $\twiall =
	\bigcup_{\lic \geq 0} \imconrds^\lic(\emptyset)$. We need to show that 
	\begin{enumerate}[1)]
		\item $\twi$ is a model of $\all{\dprg} \setminus \rejwss{\dprg, \twi}$;
			
		\item For every $\olit \in \twi$ there exists some rule $\rl \in
			\all{\dprg} \setminus \rejrds{\dprg, \twiall}$ such that $\hrl = \olit$,
			$\twi \ent \brl$ and $\lm(\hrl) > \lmmax(\brl)$.
	\end{enumerate}
	We address each point separately.
	\begin{enumerate}[1)]
		\item Take some $\lia < \lng$ and some rule $\rla_0 \in \prg_\lia$ such that
			$\twi \nent \rla_0$, i.e.\ $\twi \ent \brl[\rla_0]$ and $\twi \nent
			\hrl[\rla_0]$. Our
			goal is to show that $\rla_0$ is rejected in $\rejwss{\dprg, \twi}$, i.e.\
			\begin{equation}
				\label{eq:lemma:extended rd is extended ws:1}
				\exists \lib > \lia \; \exists \rlb \in \prg_\lib
				:
				\hrlb \in \con{\hrl[\rla_0]}
				\land
				\twi \ent \brlb
				\land
				\lmmin\!\left(\con{\hrl[\rla_0]}\right) > \lmmax(\brlb)
				\enspace.
			\end{equation}
			Note that since $\twi \nent \hrl[\rla_0]$, it follows that
			$\lcmp{\hrl[\rla_0]} \in
			\twiall$. This guarantees the existence of a literal $\lit \in
			\con{\hrl[\rla_0]}$ such that $\lit \in \twiall$ and $\lm(\lit) =
			\lmmin(\con{\hrl[\rla_0]}) = \lic + 1$ for some $\lic \geq 0$. Put $\slit =
			\imconrds^\lic(\emptyset)$. By the definition of $\lm$, $\lit$ belongs
			to $\imconrds(\slit)$. Recall that
			\begin{align*}
				\imconrds(\slit)
				&=
				\bigl\{\,
					\hrla
					\mid
					\rla \in \left(
						\rem{\dprg, \twiall}
						\cup
						\defl{\twi}
					\right)
					\land
					\brla \subseteq \slit
				\\ & \hspace{0.0cm}
					\land
					\lnot \left(
						\exists \rlb \in \rem{\dprg, \slit}
						:
						\hrlb \in \con{\hrla}
						\land
						\brlb \subseteq \twiall
					\right)
				\,\bigr\}.
			\end{align*}
			Since $\hrl[\rla_0] \in \con{\lit}$ and $\brl[\rla_0] \subseteq
			\twiall$, we conclude that $\rla_0$ belongs to $\rejrds{\dprg, \slit}$.
			Thus,
			\[
				\exists \lib > \lia
				\;
				\exists \rlb \in \prg_\lib
				:
				\hrlb \in \con{\hrl[\rla_0]}
				\land
				\brlb \subseteq \slit
				\enspace.
			\]
			It remains only to observe that $\slit \subseteq \twiall$, so $\twi \ent
			\brlb$, and that due to the fact that $\brlb \subseteq \slit =
			\imconrds^\lic(\emptyset)$,
			\[
				\lmmax(\brlb)
				\leq
				\lic
				<
				\lic + 1
				=
				\lm(\lit)
				\leq
				\lmmin\!\left(\con{\hrl[\rla_0]}\right)
				\enspace.
			\]
			
		\item Take some $\olit \in \twi$ and let $\lic \geq 0$ be such that
			$\lm(\olit) = \lic + 1$. Put $\slit = \imconrds^\lic(\emptyset)$. It
			follows that $\olit \in \imconrds(\slit)$, so there is some rule $\rl
			\in (\rem{\dprg, \twiall} \cup \defl{\twi})$ such that $\hrl = \olit$
			and $\brl \subseteq \slit$. Since $\olit$ is an objective literal, it
			follows that $\rl \notin \defl{\twi}$, so
			\[
				\rl \in \rem{\dprg, \twiall} = \all{\dprg} \setminus \rejrds{\dprg,
				\twiall}
				\enspace.
			\]
			It remains only to observe that $\slit \subseteq \twiall$, so $\twi \ent
			\brl$, and that due to the fact that $\brl \subseteq \slit =
			\imconrds^\lic(\emptyset)$,
			\[
				\lmmax(\brl)
				\leq
				\lic < \lic + 1
				=
				\lm(\olit)
				=
				\lm(\hrl)
				\enspace.
			\]
	\end{enumerate}
\end{proof}

\begin{theorem*}
	{thm:extended ws coincides with extended rd}
	Let $\dprg$ be a DLP. Then, $\modwss{\dprg} = \modrds{\dprg}$.
\end{theorem*}
\begin{proof}
	\label{proof:thm:extended ws coincides with extended rd}
	Follows from Lemmas~\ref{lemma:extended ws is extended rd} and
	\ref{lemma:extended rd is extended ws}.
\end{proof}

\begin{theorem*}
	{thm:extended semantics coincide with regular}
	Let $\dprg$ be a DLP without strong negation. Then,
	\[
		\modwss{\dprg} = \modrds{\dprg} = \modws{\dprg} = \modrd{\dprg}
		\enspace.
	\]
\end{theorem*}
\begin{proof}
	\label{proof:thm:extended semantics coincide with regural}
	Due to Thm.~\ref{thm:extended ws coincides with extended rd} and
	Prop.~\ref{prop:ws coincides with rd}, it suffices to prove that
	$\modws{\dprg} = \modwss{\dprg}$. Given that $\dprg$ does not contain
	default negation, it can be readily seen that for any interpretation $\twi$
	and level mapping $\lm$,
	\[
		\rejws{\dprg, \twi} = \rejwss{\dprg, \twi}
		\enspace.
	\]
	Thus, $\twi$ is a model of $\all{\dprg} \setminus \rejws{\dprg, \twi}$ if
	and only if it is a model of $\all{\dprg} \setminus \rejwss{\dprg, \twi}$.

	Take some interpretation $\twi$ such that $\twi$ is a model of $\all{\dprg}
	\setminus \rejws{\dprg, \twi}$. It remains to verify that $\atm \in \twi$ is
	well-supported in $\all{\dprg} \setminus \rejws{\dprg, \twi}$ if and only if
	it is well-supported in $\rem{\dprg, \twiall}$. For the direct implication,
	suppose that $\rl \in \all{\dprg} \setminus \rejws{\dprg, \twi}$ is such
	that $\hrl = \atm$, $\twi \ent \brl$ and $\lm(\hrl) > \lmmax(\brl)$. If $\rl
	\in \prg_\lia$ is rejected in $\rejrds{\dprg, \twiall}$, then there must be
	the maximal $\lib > \lia$ and a rule $\rlb \in \prg_\lib$ such that $\hrlb =
	\lpnot \hrla$ and $\twi \ent \brlb$. Consequently, $\twi \nent \rlb$, so
	$\rlb$ must itself be rejected in $\rejws{\dprg, \twi}$ and if we take the
	rejecting rule $\rlb'$ from $\prg_{\lib'}$ with $\lib' > \lib$, we find that
	$\rlb'$ does not belong to $\rejrds{\dprg, \twiall}$ (due to the maximality
	of $\lib$) and provides support for $\atm$.

	The converse implication follows immediately from the fact that
	$\rejws{\dprg, \twi}$ is a subset of $\rejrds{\dprg, \twiall}$.
\end{proof}

\begin{theorem*}
	{thm:extended ws other properties}
	The extended \RD-semantics and extended \WS-semantics satisfy all properties
	listed in Table~\ref{tab:rule update properties}.
\end{theorem*}
\begin{proof}
	\label{proof:thm:extended ws other properties}
	We prove each property for the extended \WS-semantics. For the extended
	\RD-semantics, the properties follow from Theorem~\ref{thm:extended ws
	coincides with extended rd}.
	\begin{description}
		\item[Generalisation of stable models:]
			Let $\prg$ be a program. For any interpretation $\twi$ and level mapping
			$\lm$, $\rejwss{\seq{\prg}, \twi} = \rejrds{\seq{\prg}, \twiall} =
			\emptyset$, so
			\[
				\all{\seq{\prg}} \setminus \rejwss{\seq{\prg}, \twi}
				=
				\rem{\seq{\prg}, \twiall}
				=
				\prg
				\enspace.
			\]
			Hence, $\twi$ belongs to $\modwss{\seq{\prg}}$ if and only if it belongs
			to $\modws{\prg}$. The remainder follows from Prop.~\ref{prop:extended
			ws coincides with sm}.
			\medskip

		\item[Primacy of new information:]
			Let $\dprg = \seq{\prg_\lia}_{\lia < \lng}$ be a DLP and $\twi \in
			\modwss{\dprg}$. It follows from the definition of $\rejwss{\dprg,
			\twi}$ that $\prg_{\lng - 1}$ is included in $\all{\dprg} \setminus
			\rejwss{\dprg, \twi}$. Consequently, $\twi$ is a model of $\prg_{\lng -
			1}$.
			\medskip

		\item[Fact update:]
			Let $\dprg = \seq{\prg_\lia}_{\lia < \lng}$ be a sequence of consistent
			sets of facts. It follows that regardlessly of $\twi$ and $\lm$,
			\begin{align*}
				\rejwss{\dprg, \twi}
				&=
				\rejrds{\dprg, \twiall}
				= \\
				&\hspace{-0.5cm}=\{
					(\lit.) \in \prg_\lia
					|
					\lia < \lng
					\land
					\exists \lib > \lia
					\;
					\exists \rlb \in \prg_\lib
					:
					\hrlb \in \con{\lit}
				\}.
			\end{align*}
			Thus,
			\begin{align*}
				\all{\dprg} \setminus \rejwss{\dprg, \twi}
				&=
				\rem{\dprg, \twiall}
				=\\
				&\hspace{-1.6cm}=\{
					(\lit.) \in \prg_\lia
					|
					\lia < \lng
					\land
					\forall \lib > \lia
					\;	
					\forall \rlb \in \prg_\lib
					:
					\hrlb \notin \con{\lit}
				\}.
			\end{align*}
			Put
			\begin{align*}
				\twi
				&=
				\{
					\olit \in \olits |
					\exists \lia < \lng : (\olit.) \in \prg_\lia
					\land\\
					&\hspace{3.2cm}(\forall \lib > \lia :
						\Set{\lopp{\olit}., \lcmp{\olit}.} \cap \prg_\lib = \emptyset
					)
				\}.
			\end{align*}
			From the assumption that $\prg_\lia$ is consistent for every $\lia <
			\lng$ it follows that $\twi$ is the single model of $\all{\dprg}
			\setminus \rejwss{\dprg, \twi}$ in which every objective literal is
			supported by a fact from $\rem{\dprg, \twiall}$.
			\medskip

		\item[Support:]
			Follows immediately by the definition of $\modwss{\cdot}$.
			\medskip

		\item[Idempotence:]
			Let $\prg$ be a program. It is not difficult to verify that the
			following holds for any interpretation $\twi$ and level mapping $\lm$:
			\begin{align*}
				\all{\seq{\prg, \prg}} \setminus \rejwss{\seq{\prg, \prg}, \twi}
				&= 
				\all{\seq{\prg}} \setminus \rejwss{\seq{\prg}, \twi}
				=
				\prg
				\enspace,
				\\
				\rem{\seq{\prg, \prg}, \twiall}
				&=
				\rem{\seq{\prg}, \twiall}
				=
				\prg
				\enspace.
			\end{align*}
			Thus, $\twi$ belongs to $\modwss{\seq{\prg}}$ if and only if it belongs
			to $\modwss{\seq{\prg, \prg}}$.
			\medskip

		\item[Absorption:]
			Follows from \textbf{Augmentation}.
			\medskip

		\item[Augmentation:]
			Let $\prga$, $\prgu$, $\prgv$ be programs such that $\prgu \subseteq
			\prgv$. It is not difficult to verify that the following holds for any
			interpretation $\twi$ and level mapping $\lm$:
			\begin{align*}
				\all{\seq{\prga, \prgu, \prgv}}
				\setminus
				\rejwss{\seq{\prga, \prgu, \prgv}, \twi}
				&= \\
				&\hspace{-1.4cm}=
				\all{\seq{\prga, \prgv}}
				\setminus
				\rejwss{\seq{\prga, \prgv}, \twi},
				\\
				\rem{\seq{\prga, \prgu, \prgv}, \twiall}
				&=
				\rem{\seq{\prga, \prgv}, \twiall}
				\enspace.
			\end{align*}
			Thus, $\twi$ belongs to $\modwss{\seq{\prga, \prgu, \prgv}}$ if and only
			if it belongs to $\modwss{\seq{\prga, \prgv}}$.
			\medskip

		\item[Non-interference:]
			Let $\prga$, $\prgu$, $\prgv$ be programs such that $\prgu$ and $\prgv$
			are over disjoint alphabets. It is not difficult to verify that the
			following holds for any interpretation $\twi$ and level mapping $\lm$:
			\begin{align*}
				\all{\seq{\prga, \prgu, \prgv}}
				\setminus
				\rejwss{\seq{\prga, \prgu, \prgv}, \twi}
				&= \\
				&\hspace{-2.2cm}=
				\all{\seq{\prga, \prgv, \prgu}}
				\setminus
				\rejwss{\seq{\prga, \prgv, \prgu}, \twi},
				\\
				\rem{\seq{\prga, \prgu, \prgv}, \twiall}
				&=
				\rem{\seq{\prga, \prgv, \prgu}, \twiall}
				.
			\end{align*}
			Thus, $\twi$ belongs to $\modwss{\seq{\prga, \prgu, \prgv}}$ if and only
			if it belongs to $\modwss{\seq{\prga, \prgv, \prgu}}$.
			\medskip

		\item[Immunity to empty updates:]
			Let $\seq{\prg_\lia}_{\lia < \lng}$ be a DLP such that
			$\prg_\lib = \emptyset$. It is not difficult to verify that the
			following holds for any interpretation $\twi$ and level mapping $\lm$:
			\begin{align*}
				\all{\seq{\prg_\lia}_{\lia < \lng}}
				\setminus
				\rejwss{\seq{\prg_\lia}_{\lia < \lng}, \twi}
				&= \\
				&\hspace{-2.6cm}=
				\all{\seq{\prg_\lia}_{\lia < \lng \land \lia \neq \lib}}
				\setminus
				\rejwss{\seq{\prg_\lia}_{\lia < \lng \land \lia \neq \lib}, \twi},
				\\
				\rem{\seq{\prg_\lia}_{\lia < \lng}, \twiall}
				&=
				\rem{\seq{\prg_\lia}_{\lia < \lng \land \lia \neq \lib}, \twiall}
				.
			\end{align*}
			Thus, $\twi$ belongs to $\modwss{\seq{\prg_\lia}_{\lia < \lng}}$ if and
			only if it belongs to $\modwss{\seq{\prg_\lia}_{\lia < \lng \land \lia
			\neq \lib}}$.
			\medskip

		\item[Immunity to tautologies:]
			Let $\seq{\prg_\lia}_{\lia < \lng}$ be a DLP and $\seq{\prgb_\lia}_{\lia
			< \lng}$ is a sequence of sets of tautologies. It follows from basic
			properties of level mappings that for any interpretation $\twi$ and
			level mapping $\lm$, the sets
			\begin{align*}
				& \all{\seq{\prg_\lia}_{\lia < \lng}} 
					\setminus
					\rejwss{\seq{\prg_\lia}_{\lia < \lng}, \twi}
				\enspace \text{and}\\
				& \all{\seq{\prg_\lia \cup \prgb_\lia}_{\lia < \lng}} 
					\setminus
					\rejwss{\seq{\prg_\lia \cup \prgb_\lia}_{\lia < \lng}, \twi}
			\end{align*}
			differ only in the presence or absence of tautologies. Similarly, the
			sets
			\begin{align*}
				& \rem{\seq{\prg_\lia}_{\lia < \lng}, \twiall}
				&& \text{and}
				&& \rem{\seq{\prg_\lia \cup \prgb_\lia}_{\lia < \lng}, \twiall}
			\end{align*}
			differ only in the presence or absence of tautologies. Consequently,
			\[
				\twi
				\ent
				\all{\seq{\prg_\lia}_{\lia < \lng}} 
				\setminus
				\rejwss{\seq{\prg_\lia}_{\lia < \lng}, \twi}
			\]
			if and only if
			\[
				\twi
				\ent
				\all{\seq{\prg_\lia \cup \prgb_\lia}_{\lia < \lng}} 
				\setminus
				\rejwss{\seq{\prg_\lia \cup \prgb_\lia}_{\lia < \lng}, \twi}
				\enspace.
			\]
			Furthermore, the extra tautological rules in $\rem{\seq{\prg_\lia \cup
			\prgb_\lia}_{\lia < \lng}, \twiall}$ cannot provide well-support for any
			literal, so $\twi$ is well-supported by $\rem{\seq{\prg_\lia}_{\lia <
			\lng}, \twiall}$ if and only if it is well-supported by
			$\rem{\seq{\prg_\lia \cup \prgb_\lia}_{\lia < \lng}, \twiall}$.  Thus,
			$\twi$ belongs to $\modwss{\seq{\prg_\lia}_{\lia < \lng}}$ if and only
			if it belongs to $\modwss{\seq{\prg_\lia \cup \prgb_\lia}_{\lia <
			\lng}}$.

		\item[Causal rejection principle:]
			Follows directly from the definition of $\rejwss{\dprg, \twi}$ and of
			$\modwss{\dprg}$.
	\end{description}
\end{proof}

\begin{theorem*}
	{thm:complexity}
	Let $\dprg$ be a DLP. The problem of deciding whether some $\twi \in
	\modwss{\dprg}$ exists is \NP-complete. Given a literal $\lit$, the problem
	of deciding whether for all $\twi \in \modwss{\dprg}$ it holds that $\twi
	\ent \lit$ is \coNP-complete.
\end{theorem*}
\begin{proof}
	\label{proof:thm:complexity}
	Hardness of these decision problems follows from the property
	\textbf{Generalisation of stable models} (c.f.\ Table~\ref{tab:rule update
	properties} and Thm.~\ref{thm:extended ws other properties}).

	In case of deciding whether some $\twi \in \modwss{\dprg}$ exists,
	membership to \NP{} follows from this non-deterministic procedure that runs in
	polynomial time:
	\begin{enumerate}[1.]
		\item Guess an interpretation $\twi$ and a level mapping $\lm$;

		\item Verify deterministically in polynomial time that $\twi$ is an
			extended \WS-model of $\dprg$ w.r.t.\ $\lm$. If it is, return ``true'',
			otherwise return ``false''.
	\end{enumerate}

	Similarly, deciding whether for all $\twi \in \modwss{\dprg}$ it holds that
	$\twi \ent \lit$ can be done in \coNP{} since the complementary problem of
	deciding whether $\twi \nent \lit$ for some $\twi \in \modwss{\dprg}$
	belongs to \NP{}, as verified by the following non-deterministic polynomial
	algorithm:
	\begin{enumerate}[1.]
		\item Guess an interpretation $\twi$ and a level mapping $\lm$;

		\item Verify deterministically in polynomial time that $\twi$ is an
			extended \WS-model of $\dprg$ and that $\twi \nent \lit$. If this is the
			case, return ``true'', otherwise return ``false''. 
	\end{enumerate}
\end{proof}

\begin{lemma}
	\label{lemma:generalised early recovery:1}
	Let $\dprg = \seq{\prg_\lia}_{\lia < \lng}$ be a DLP such that $\all{\dprg}$
	is an acyclic program w.r.t.\ the level mapping $\lm$, $\twi_0 =
	\emptyset$, for all $\lic \geq 0$, $\twi_{\lic + 1}$ be the set of objective
	literals
\begin{align*}
		\{
			&\hrla \in \olits
			|
			\rl \in \prg_\lia
			\land
			\lm(\hrla) \leq \lic + 1
			\land
			\twi_\lic \ent \brla\\
		& \hspace{1.5cm}\land		
			\lnot
			\left(
				\exists \lib > \lia
				\;
				\exists \rlb \in \prg_\lib
				:
				\hrlb \in \con{\hrla}
				\land
				\twi_\lic \ent \brlb
			\right)
		\}
\end{align*}
	and $\twi = \bigcup_{\lic \geq 0} \twi_\lic$. For every objective literal
	$\olit$ with $\lm(\olit) = \lic_0$ and all $\lic$ such that $\lic \geq
	\lic_0$ the following holds:
	\begin{align*}
		& \olit \in \twi_\lic
		&& \text{if and only if}
		&& \olit \in \twi_{\lic_0}
		\enspace.
	\end{align*}
\end{lemma}
\begin{proof}
	We prove by induction on $\lic_0$:
	\begin{enumerate}[1$^\circ$]
		\item For $\lic_0 = 0$ this follows from the assumption that $\all{\dprg}$
			is acyclic w.r.t.\ $\lm$: since $\lm(\olit) = \lm(\lpnot \olit) = 0$,
			any rule in $\all{\dprg}$ with either $\olit$ or $\lpnot \olit$ in its
			head would have to have a body with a negative level, which is not
			possible.
			
		\item Suppose that the claim holds for all $\lic_0' \leq \lic_0$, we will
			prove it for $\lic_0 + 1$. Take an objective literal $\olit$ with
			$\lm(\olit) = \lic_0 + 1$ and some $\lic \geq \lic_0$. We need to show
			that $\olit \in \twi_{\lic_0 + 1}$ holds if and only if $ \olit \in
			\twi_{\lic + 1}$. Note that $\olit \in \twi_{\lic_0 + 1}$ holds if and
			only if for some $\lia < \lng$ and some $\rl \in \prg_\lia$,
	\begin{align*}
				\hrla &= \olit
				\land
				\twi_{\lic_0} \ent \brla
				\land\\
				&\hspace{1.0cm}\lnot
				\left(
					\exists \lib > \lia
					\;
					\exists \rlb \in \prg_\lib
					:
					\hrlb \in \con{\hrla}
					\land
					\twi_{\lic_0} \ent \brlb
				\right).
	\end{align*}
			Our assumption that $\all{\dprg}$ is acyclic w.r.t.\ $\lm$
			together with the inductive assumption entail that we can
			equivalently write
	\begin{align*}
				\hrla &= \olit
				\land
				\twi_{\lic} \ent \brla
				\land\\
				&\hspace{1.0cm}\lnot
				\left(
					\exists \lib > \lia
					\;
					\exists \rlb \in \prg_\lib
					:
					\hrlb \in \con{\hrla}
					\land
					\twi_{\lic} \ent \brlb
				\right),
	\end{align*}
			which is equivalent to $\olit \in \twi_{\lic + 1}$. 
	\end{enumerate}
\end{proof}

\begin{lemma}
	\label{lemma:generalised early recovery:2}
	Let $\dprg = \seq{\prg_\lia}_{\lia < \lng}$ be a DLP such that $\all{\dprg}$
	is an acyclic program w.r.t.\ the level mapping $\lm$, $\twi_0 =
	\emptyset$, for all $\lic \geq 0$, $\twi_{\lic + 1}$ be the set of objective
	literals
	\begin{align*}
		\{
			&\hrla \in \olits
			|
			\rl \in \prg_\lia
			\land
			\lm(\hrla) \leq \lic + 1
			\land
			\twi_\lic \ent \brla
				\land\\
				&\hspace{2.0cm}\lnot
			\left(
				\exists \lib > \lia
				\;
				\exists \rlb \in \prg_\lib
				:
				\hrlb \in \con{\hrla}
				\land
				\twi_\lic \ent \brlb
			\right)
		\}
	\end{align*}
	and $\twi = \bigcup_{\lic \geq 0} \twi_\lic$. For every literal $\lit$ with
	$\lm(\lit) = \lic_0$ and all $\lic \geq \lic_0$, the following holds:
	\begin{align*}
		& \twi \ent \lit
		&& \text{if and only if}
		&& \twi_{\lic} \ent \lit
		\enspace.
	\end{align*}
\end{lemma}
\begin{proof}
	Take some literal $\lit$ with $\lm(\lit) = \lic_0$ and $\lic \geq \lic_0$.
	We consider two cases:
	\begin{enumerate}[a)]
		\item If $\lit$ is an objective literal $\olit$, then $\twi \ent \lit$
			holds if and only if for some $\lic_1 \geq 0$, $\olit \in
			\twi_{\lic_1}$. It follows from the definition of $\twi_\lic$ that for
			$\lic < \lic_0$ this cannot be the case, so $\twi \ent \lit$ holds if
			and only if for some $\lic_1 \geq \lic_0$, $\olit \in \twi_{\lic_1}$. By
			Lemma~\ref{lemma:generalised early recovery:1}, this is equivalent to
			$\twi_\lic \ent \lit$.

		\item If $\lit$ is a default literal $\lpnot \olit$, then $\twi \ent \lit$
			holds if and only if for all $\lic_1 \geq 0$, $\olit \notin
			\twi_{\lic_1}$. Due to the definition of $\twi_{\lic_1}$, for $\lic_1 <
			\lic_0$ this is guaranteed, so $\twi \ent \lit$ holds if and only if for
			all $\lic_1 \geq \lic_0$, $\olit \notin \twi_{\lic_1}$. By
			Lemma~\ref{lemma:generalised early recovery:1}, this is equivalent to
			$\twi_\lic \ent \lit$. 
	\end{enumerate}
\end{proof}

\begin{theorem*}
	{thm:generalised early recovery}
	The extended \RD-semantics and extended \WS-semantics satisfy the
	generalised early recovery principle.
\end{theorem*}
\begin{proof}
	\label{proof:thm:generalised early recovery}
	Let $\dprg = \seq{\prg_\lia}_{\lia < \lng}$ be a DLP such that $\all{\dprg}$
	is an acyclic program w.r.t.\ the level mapping $\lm$, and let
	$\seq{\twi_\lic}_{\lic \geq 0}$ and $\twi$ be as in
	Lemma~\ref{lemma:generalised early recovery:2}. Our goal is to show that
	$\twi$ is an extended \WS-model of $\dprg$ w.r.t.\ $\lm$, i.e.\ we need to
	verify the following three statements:
	\begin{enumerate}[1)]
		\item $\twi$ is a consistent set of objective literals, i.e.\ it is an
			interpretation;

		\item $\twi$ is a model of $\all{\dprg} \setminus \rejwss{\dprg, \twi}$;

		\item For every objective literal $\olit \in \twi$ there exists some rule
			$\rl \in \rem{\dprg, \twiall}$ such that $\hrl = \olit$, $\twi \ent
			\brl$ and $\lm(\hrl) > \lmmax(\brl)$.
	\end{enumerate}
	We prove each statement separately.
	\begin{enumerate}[1)]
		\item To show that $\twi$ is a consistent set of objective literals,
			suppose that for some $\olit \in \olits$, both $\olit$ and $\lnot \olit$
			belong to $\twi$. Also, suppose that $\lm(\olit) = \lic$. By
			Lemma~\ref{lemma:generalised early recovery:2} we conclude that
			$\twi_\lic$ contains both $\olit$ and $\lpnot \olit$. Thus, by the
			definition of $\twi_\lic$, for some $\lia < \lng$ there must exist rules
			$\rla, \rlb \in \prg_\lia$ such that $\hrla = \olit$, $\hrlb = \lpnot
			\olit$, $\twi \ent \brla$ and $\twi \ent \brlb$. But then we obtain a
			conflict with the assumption that all conflicts in $\dprg$ are solved
			since it follows that for some $\lib > \lia$ there is a fact $\rlb' \in
			\prg_\lib$ such that either $\hrl[\rlb'] \in \con{\hrla}$ or
			$\hrl[\rlb'] \in \con{\hrlb}$.
			
		\item In order to prove that $\twi$ is a model of $\all{\dprg} \setminus
			\rejwss{\dprg, \twi}$, take some rule
			\[
				\rla
				\in
				\all{\dprg}
				\setminus
				\rejwss{\dprg, \twi}
			\]
			and assume that $\twi \ent \brla$. Let $\lm(\hrla) = \lic_0$.
			We consider two cases:
			\begin{enumerate}[a)]
				\item If $\hrla$ is an objective literal $\olit$, then it follows from
					the definition of $\twi_\lic$, the definition of $\rejwss{\dprg,
					\twi}$ and Lemma~\ref{lemma:generalised early recovery:2} that
					$\olit \in \twi$. Thus, $\twi \ent \hrla$.

				\item If $\hrla$ is a default literal $\lpnot \olit$, then it follows
					from the definition of $\twi_\lic$, definition of $\rejwss{\dprg,
					\twi}$, the assumption that all conflicts in $\dprg$ are solved and
					Lemma~\ref{lemma:generalised early recovery:2} that $\olit \notin
					\twi$. Thus, $\twi \ent \hrla$.
			\end{enumerate}

		\item Finally, we need to demonstrate that for every $\olit \in \twi$
			there exists some rule $\rl \in \rem{\dprg, \twiall}$ such that $\hrl =
			\olit$, $\twi \ent \brl$ and $\lm(\hrl) > \lmmax(\brl)$. This follows
			from the definition of $\twi$ and of $\rejrds{\dprg, \twiall}$.
	\end{enumerate}
\end{proof}
